\documentclass{article}

\usepackage{microtype}
\usepackage{graphicx}
\usepackage{subcaption}
\usepackage{booktabs} 

\usepackage{hyperref}

\usepackage[preprint]{icml2026}

\usepackage{amsmath}
\usepackage{amssymb}
\usepackage{mathtools}
\usepackage{amsthm}

\usepackage[capitalize,noabbrev]{cleveref}
\usepackage{booktabs}
\usepackage{multirow}
\usepackage{caption}
\usepackage{graphicx}
\usepackage{algorithm}
\usepackage{algorithmic}

\usepackage{xy}
\xyoption{all}

\def\mA{{\boldsymbol{A}}}
\def\mB{{\boldsymbol{B}}}
\def\mI{{\boldsymbol{I}}}
\def\mW{{\boldsymbol{W}}}
\def\mM{{\boldsymbol{M}}}

\theoremstyle{plain}
\newtheorem{theorem}{Theorem}[section]

\newtheorem{lemma}[theorem]{Lemma}
\newtheorem{corollary}[theorem]{Corollary}
\theoremstyle{definition}
\newtheorem{definition}[theorem]{Definition}

\newtheorem{example}[theorem]{Example}
\theoremstyle{remark}

\usepackage[textsize=tiny]{todonotes}

\icmltitlerunning{Modeling Topological Impact on Node Attribute Distributions in Attributed Graphs}

\begin{document}

\twocolumn[
  \icmltitle{Modeling Topological Impact on Node Attribute Distributions in Attributed Graphs}



  \icmlsetsymbol{equal}{*}

  \begin{icmlauthorlist}
    \icmlauthor{Amirreza Shiralinasab Langari}{yyy}
    \icmlauthor{Leila Yeganeh}{yyy}
    \icmlauthor{Kim Khoa Nguyen}{yyy}
  \end{icmlauthorlist}

  \icmlaffiliation{yyy}{Department of Electrical Engineering, \'Ecole de Technologie Sup\'erieure (ETS), University of Quebec, Montreal, Canada}

  \icmlcorrespondingauthor{Amirreza Shiralinasab Langari}{amirreza.shiralinasab-langari.1@ens.etsmtl.ca}

  \icmlkeywords{Machine Learning, ICML}

  \vskip 0.3in
]



\printAffiliationsAndNotice{}  

\begin{abstract}
  We investigate how the topology of attributed graphs influences the distribution of node attributes. This work offers a novel perspective by treating topology and attributes as structurally distinct but interacting components. We introduce an algebraic approach that combines a graph's topology with the probability distribution of node attributes, resulting in topology-influenced distributions. First, we develop a categorical framework to formalize how a node perceives the graph's topology. We then quantify this point of view and integrate it with the distribution of node attributes to capture topological effects. We interpret these topology-conditioned distributions as approximations of the posteriors $P(\cdot \mid v)$ and $P(\cdot \mid \mathcal{G})$.
  We further establish a principled sufficiency condition by showing that, on complete graphs, where topology carries no informative structure, our construction recovers the original attribute distribution. To evaluate our approach, we introduce an intentionally simple testbed model, \textbf{ID}, and use unsupervised graph anomaly detection as a probing task.
\end{abstract}
\section{Introduction}
\begin{figure}[t]
    \centering
  \includegraphics[scale=0.5]{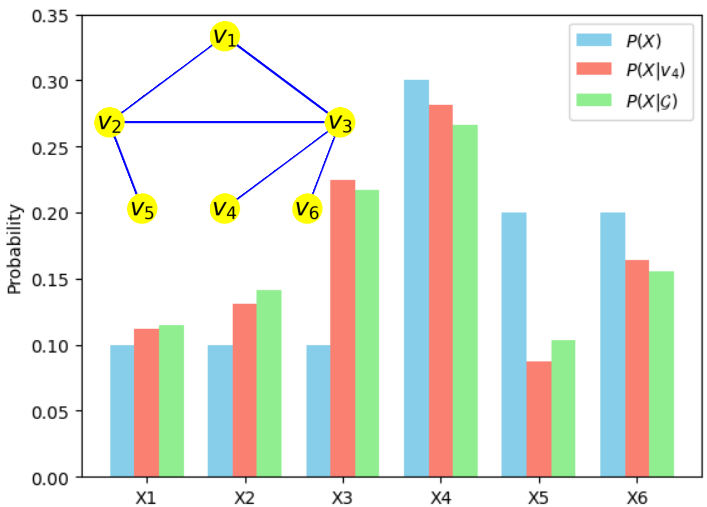}
   \caption{Illustration of an attributed graph $\mathcal{G} = (V, E, X)$ alongside the probability distribution of node attributes $X$. The inset graph represents the topology of $\mathcal{G}$, while the bar chart compares the prior distribution $P(X)$ (blue), the topology-influenced distribution from the point of view of node $v_4$, $P(X \mid v_4)$ (red), and the overall topology-conditioned distribution $P(X \mid \mathcal{G})$ (green). These distributions highlight how the topology of $\mathcal{G}$ shapes the points of view of $v_4$ regarding $P$.}
    \label{affdis}
\end{figure}

Attributed graphs contain both structural and attribute information, yet the influence of a graph's topology on the distribution of node attributes remains underexplored \citep{unquan, confpre,graphpost}. Understanding this interaction is crucial for downstream tasks such as representation learning and anomaly detection. In this paper, we propose an algebraic framework to quantify how a graph’s topology affects the distribution of node attributes. By integrating topology and distribution of node attributes, our approach provides enriched representations of graphs that benefit learning tasks.

Formally, consider an attributed graph $\mathcal{G} = (V, E, X)$, where $X$ denotes the set of node attributes governed by a prior distribution $P$. By treating the graph’s topology as evidence, we aim to derive the posterior distributions $P(\cdot \mid v)$ and $P(\cdot \mid \mathcal{G})$-capturing, respectively, the distribution of node attributes from the perspective of an individual node $v$ and the whole graph. This formulation allows us to model how topology affects attribute distribution, facilitating a principled fusion of structural and attribute information.

This integration is particularly relevant for \textit{Graph Anomaly Detection (GAD)}~\citep{bond, gad-nr}, where the goal is to detect anomalies based on deviations from an expected distribution of node attributes in the context of the graph. As illustrated in Figure \ref{affdis}, the presence of topology significantly alters these distributions. While node attributes $X_1, X_2$, and $X_3$ are equally probable in the prior, their probabilities diverge under the influence of topology, showing how $P(\cdot \mid \mathcal{G})$ encodes new information unavailable in $P$ alone. For this reason, we use GAD as a probing testbed: success in this task critically depends on whether topology-conditioned distributions are meaningfully captured. While our experiments focus on anomaly detection, our goal is not to introduce a new GAD architecture, but to evaluate a general framework for topology-conditioned probabilistic reasoning, which is applicable beyond this setting.

Central to our approach is the notion that each node in a graph has its own \textit{point of view}, a topologically centered perception of the graph structure. This perspective can be imagined as a seed planted at the node's location, from which the graph's topology radiates outward. Figure \ref{poin_of_view_grid} illustrates this concept by depicting the point of view of a node $v \in \mathcal{G}$ within the graph's topology. By formalizing and quantifying this idea, we can derive $P(\cdot \mid v)$ for each node and aggregate them into a global distribution $P(\cdot \mid \mathcal{G})$.

This idea echoes a well-known theme in mathematics: studying an object through its relationships with others. In \textit{category theory}, this is formalized using the notion of an \textit{under category}, which captures how an object ``views'' other objects via all morphisms originating from it~\cite{abstractandconcretecategories}. For example, in \textit{commutative algebra}~\cite{commutativealgebra}, one often studies rings from the point of view of a specific ring, leading to the development of the category of commutative algebras. Inspired by this precedent, we adopt a categorical perspective to study node viewpoints within a graph.

To formalize the points of view of nodes, we begin by associating a category with each graph. The classical \textit{free category}~\cite{maclane} generated by a graph, where nodes are objects and paths are morphisms, provides a natural way to represent the graph’s topology. However, its morphisms are purely structural and lack any quantifiable interpretation. To overcome this, we build upon the recently introduced \textit{Grothendieck Graph Neural Networks (GGNN) framework}~\cite{GGNN}, which provides an algebraic foundation for working with graphs. Within this framework, we introduce an enhanced version of the free category by defining monoidal elements that mirror paths, allowing their matrix representations to be derived through the monoidal homomorphism $\mathsf{Tr}$. This construction allows us to assign quantitative meaning to morphisms, enabling a more expressive representation of graph topology tailored to our probabilistic objectives. Our contributions are as follows:

\textbf{Topological Viewpoint Formalization:} We introduce a novel category associated with each graph that determines the graph up to isomorphism. Using \textit{under categories} derived from this structure, we formalize the point of view of each node regarding the graph’s topology.
    
 \textbf{Posterior Distribution Modeling:} We propose a method to compute the posterior distributions $P(\cdot \mid v)$ and $P(\cdot \mid \mathcal{G})$ by merging each node’s point of view regarding the graph’s topology with a probability distribution $P$ on node attributes. This is achieved by defining weights for directed edges induced by $P$.

  \textbf{Sufficiency Justification:} To demonstrate the sufficiency of our approach, we assess it on complete graphs under the assumption that their fully connected topology does not affect the distribution of node attributes.
    
\textbf{Induced Distribution (ID): a Stress-Test Instantiation:}
As a secondary contribution, to isolate and evaluate the expressive power of topology-conditioned distributions, we introduce an intentionally simple testbed model, ID, and use unsupervised graph anomaly detection as a probing task.

\begin{figure}[t]
    \centering
  \includegraphics[scale=0.37]{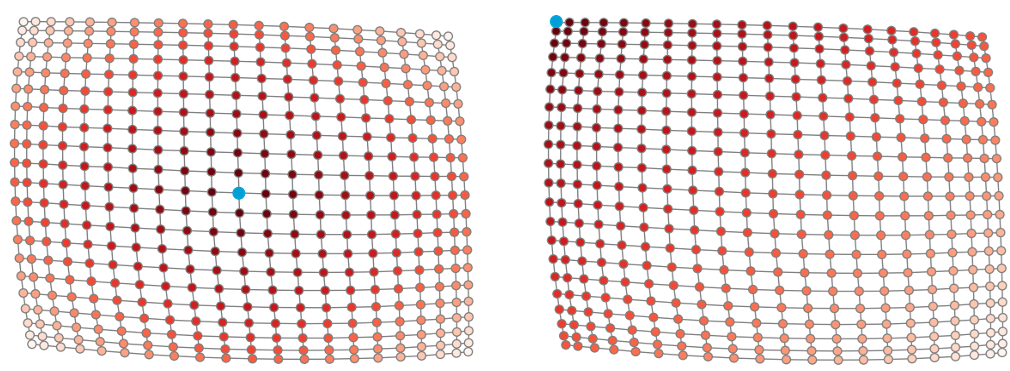}
   \caption{Visualization of the point of view of a node $v\in\mathcal G$ (highlighted in blue) with respect to the topology of the graph $\mathcal G$. The red color spectrum represents the influence or perspective of node $v$ on other nodes, with varying intensities indicating differing levels of topological proximity or relevance.}
    \label{poin_of_view_grid}
\end{figure}
\section{Preliminaries and notation}
\label{sec:formatting}

We consider undirected graphs denoted by $\mathcal{G} = (V, E)$ and attributed graphs denoted by $\mathcal{G} = (V, E, X)$, where $V$ is the set of nodes, $E$ is the set of edges, and $X$ is the set of node attributes.
Subgraphs and multigraphs are denoted by $\mathcal{D}$ and $\mathcal{M}$, respectively. A probability distribution over node attributes is denoted by $P$.
For matrices, we use uppercase bold symbols such as $\mA$. The entry at row $i$ and column $j$ is denoted $\mA_{i,j}$, the $i$-th row by $\mA_{i,:}$, and the $j$-th column by $\mA_{:,j}$.
Definitions-including \emph{monoid}, \emph{category}, \emph{under category}, and \emph{functor}-are provided in Appendix~\ref{cat}.

\paragraph{Grothendieck Graph Neural Networks (GGNN) Framework} 
The GGNN framework introduces two monoidal operations, denoted by $\bullet$ and $\circ$, along with four monoids:
{\small\[
(\mathsf{Mod}(\mathcal{G}), \bullet) \subseteq (\mathsf{SMult}(\mathcal{G}), \bullet), 
(\mathsf{Mom}(\mathcal{G}), \circ) \subseteq (\mathsf{Mat}_{|V|}(\mathbb{R}), \circ)
\]}
and a monoidal homomorphism $\mathsf{Tr}$ that maps elements of $\mathsf{Mod}(\mathcal{G})$ to their matrix representations. We summarize the relevant concepts from~\cite{GGNN} that are used throughout the paper:
\begin{definition}
Let $\mathcal{G} = (V, E)$ be an undirected graph.
\begin{itemize}
    \item[(1)] A \textbf{path} $p$ from node $v_{p_1}$ to node $v_{p_m}$ is an ordered sequence $v_{p_1}, e_{p_1}, v_{p_2}, e_{p_2}, \dots, v_{p_{m-1}}, e_{p_{m-1}}, v_{p_m}$, where $e_{p_i}$ represents an edge connecting nodes $v_{p_i}$ and $v_{p_{i+1}}$. Highlighting the edges between nodes, this definition of a path is applicable to multigraphs where multiple edges can exist between two nodes.

    \item[(2)] A \textbf{directed subgraph} $\mathcal{D}$ of $\mathcal{G}$ is a connected and acyclic subgraph of $\mathcal{G}$ where every edge of $\mathcal{D}$ is directed.

    \item[(3)] The monoid $\mathsf{SMult}(\mathcal{G})$ is the set of ordered pairs $(\mathcal{M}, S)$, where:
    \begin{itemize}
        \item $\mathcal{M} = \bigoplus_{i=1}^k \mathcal{D}_i$ is a multigraph obtained by taking the union of the sets of nodes and the disjoint union of the sets of directed edges from directed subgraphs $\mathcal{D}_i$.
        \item $S \subseteq \mathsf{Paths}(\mathcal{M})$ is a set of paths in $\mathcal{M}$.
    \end{itemize}
    The binary operation $\bullet$ on $\mathsf{SMult}(\mathcal{G})$ is defined as $(\mathcal{M}, S) \bullet (\mathcal{N}, T) = (\mathcal{M} \bigoplus \mathcal{N}, S \star T)$, where $S \star T$ is the union of $S$, $T$ (as subsets of $\mathsf{Paths}(\mathcal{M} \bigoplus \mathcal{N})$), and the collection of paths formed by composing paths in $S$ with paths in $T$.

    \item[(4)] $\mathsf{Mod}(\mathcal{G})$ is defined as the submonoid of $\mathsf{SMult}(\mathcal{G})$ generated by the set of ordered pairs $(\mathcal{D}, \mathsf{Paths}(\mathcal{D}))$, where $\mathcal{D}$ is a directed subgraph of $\mathcal{G}$ and $\mathsf{Paths}(\mathcal{D})$ is the set of all paths in $\mathcal{D}$ adhering to the directions. Where there is no risk of confusion, we denote $(\mathcal{D}, \mathsf{Paths}(\mathcal{D}))$ simply as $\mathcal{D}$. In particular, for a directed edge $e$, the element $(e, \{e\})$ of $\mathsf{Mod}(\mathcal{G})$ is abbreviated as $e$.

    \item[(5)] A \textbf{cover} for a graph $\mathcal{G}$ is a collection of finitely many elements of $\mathsf{Mod}(\mathcal{G})$.

    \item[(6)] The monoidal operation $\circ$ on $\mathsf{Mat}_n(\mathbb{R})$, the set of all $n \times n$ matrices, is defined as follows:
    \[
    \mA \circ \mB = \mA + \mB + \mA \mB.
    \]
\end{itemize}
\end{definition}

\section{Points of view of nodes to the graph's topology}

In this section, we propose a categorical framework to formalize the notion of \emph{points of view} of nodes with respect to the topology of a graph $\mathcal{G} = (V, E)$. By aggregating these viewpoints into a single element of the monoid $\mathsf{SMult}(\mathcal{G})$ and mapping this element to a matrix representation, we provide a quantitative interpretation of node-level perspectives.
\subsection{Categorical interpretation of points of view of nodes}

To represent the points of view of nodes categorically, we define a category associated with the graph $\mathcal{G}$. The objects of this category are the nodes of $\mathcal{G}$, and its morphisms are constructed from certain elements of $\mathsf{Mod}(\mathcal{G})$ as follows.

Let $e_1, \dots, e_m$ be a sequence of $m$ directed edges such that the target of $e_i$ coincides with the source of $e_{i+1}$. We call such edges \emph{composable}. Let $u$ denote the source of $e_1$ and $v$ the target of $e_m$. The monoidal element $e_1 \bullet \cdots \bullet e_m$ is then defined as a morphism from $u$ to $v$.

\begin{theorem}\label{cat(G)}
    For a graph $\mathcal{G} = (V, E)$, the structure $\mathsf{Cat}(\mathcal{G})$, whose objects are the nodes of $\mathcal{G}$ and whose morphisms are defined as above, forms a category.
\end{theorem}

This category $\mathsf{Cat}(\mathcal{G})$ determines the graph up to isomorphism, as formalized below.

\begin{theorem}\label{iso_graphs_iso_cats}
    Two graphs $\mathcal{G}$ and $\mathcal{H}$ are isomorphic if and only if their corresponding categories $\mathsf{Cat}(\mathcal{G})$ and $\mathsf{Cat}(\mathcal{H})$ are isomorphic.
\end{theorem}

This result implies that $\mathsf{Cat}(\mathcal{G})$ fully encodes the topological structure of $\mathcal{G}$. Consequently, the under categories of nodes in $\mathsf{Cat}(\mathcal{G})$ offer a principled way to describe individual node perspectives.

\paragraph{Points of view of nodes} For a node $v \in V$, we define its \emph{point of view} on the topology of the graph as the under category $v / \mathsf{Cat}(\mathcal{G})$, denoted more simply by $v / \mathcal{G}$. The objects in $v / \mathcal{G}$ correspond to composable paths originating at $v$, representing the ways in which $v$ connects to the rest of the graph. Thus, $v / \mathcal{G}$ captures a topological perspective unique to node $v$.
Moreover, these perspectives are disjoint in the following sense: for distinct nodes $v_i, v_j \in V$ with $i \ne j$, the categories $v_i / \mathcal{G}$ and $v_j / \mathcal{G}$ share no objects in common.

\subsection{Aggregation of points of view of nodes}

Although each node possesses its own point of view, the graph’s connectivity introduces shared commonalities among these perspectives. For instance, a directed edge \( e: u \rightarrow v \) induces a transformation from \( v/\mathcal{G} \) to \( u/\mathcal{G} \) via composition, resulting in a functor. This interplay between neighboring nodes’ points of view is formalized in the following theorem:

\begin{theorem}\label{functor}
    Every directed edge \( e: u \rightarrow v \) induces a functor \( e \bullet - : v/\mathcal{G} \rightarrow u/\mathcal{G} \).
\end{theorem}

It is straightforward to verify that \( (c \bullet e) \bullet - = (c \bullet -)(e \bullet -) \), which implies that any path between two nodes reflects a degree of commonality in their points of view. This observation motivates the aggregation of these perspectives based on their commonalities to obtain unique perspectives for nodes. However, working directly with the categories \( v/\mathcal{G} \) is challenging, as each contains infinitely many objects of varying length. Formalizing the shared structure among such objects is nontrivial. To address this, we adopt the notion of \emph{covers} from the GGNN framework, constructing infinitely many covers, each comprising finitely many objects from the categories \( v/\mathcal{G} \).

\paragraph{Cover of points of view}
For \( m \in \mathbb{N} \), we define the cover
{\small\[\mathsf{Cov}(m)=\left\{ e_1 \bullet \cdots \bullet e_m \,\middle|\, e_1 \bullet \cdots \bullet e_m \in\bigcup_{v \in \mathcal{G}} \mathrm{ob}(v / \mathcal{G}) \right\}\]}
where \( \mathrm{ob}(v / \mathcal{G}) \) denotes the collection of objects in \( v / \mathcal{G} \). The cover \( \mathsf{Cov}(m) \) selects a finite number of elements uniformly across nodes by fixing the path length \( m \).

By definition of the operation \( \bullet \), an element \( e_1 \bullet \cdots \bullet e_m \) can be represented as an ordered pair \( (\mathcal{M}, S) \), where \( \mathcal{M} \) is the directed multigraph obtained from the disjoint union of the directed edges \( e_i \) for \( 1 \leq i \leq m \), and \( S \subseteq \mathsf{Paths}(\mathcal{M}) \) consists of paths \( p = e_{p_1}, e_{p_2}, \ldots, e_{p_k} \) in \( \mathcal{M} \) such that \( 1 \leq p_1 < p_2 < \cdots < p_k \leq m \).

Consequently, \( e_1 \bullet \cdots \bullet e_m \) contains information from \( e_1 \bullet \cdots \bullet e_{m-1} \), and hence \( \mathsf{Cov}(m) \) captures and extends the structure of \( \mathsf{Cov}(m') \) for any \( m' \leq m \). Thus, as \( m \) increases, \( \mathsf{Cov}(m) \) inherits more information from the union of all node perspectives.
Each cover \( \mathsf{Cov}(m) \) can be aggregated into the following element:
{\small
    \[\left( \bigoplus_{(\mathcal{M}_\alpha, S_\alpha) \in \mathsf{Cov}(m)} \mathcal{M}_\alpha, \bigcup_{(\mathcal{M}_\alpha, S_\alpha) \in \mathsf{Cov}(m)} \nu(S_\alpha) \right)\]}

where \( \nu(S_\alpha) \) denotes the image of \( S_\alpha \) as a subset of \( \mathsf{Paths}\left( \bigoplus_{(\mathcal{M}_\alpha, S_\alpha) \in \mathsf{Cov}(m)} \mathcal{M}_\alpha \right) \).
To reflect structural commonalities, we quotient this element by an equivalence relation defined as follows:

\begin{definition}
    For \( e, e' \in \bigoplus_{(\mathcal{M}_\alpha, S_\alpha) \in \mathsf{Cov}(m)} \mathcal{M}_\alpha \), we define \( e \sim e' \) if and only if \( e \in \mathcal{M}_{\alpha_0} \) and \( e' \in \mathcal{M}_{\alpha_1} \) have the same source and target nodes, and both are the \( i \)-th directed edge in their respective elements \( e_1 \bullet \cdots \bullet e_m = (\mathcal{M}_{\alpha_0}, S_{\alpha_0}) \) and \( e_1' \bullet \cdots \bullet e_m' = (\mathcal{M}_{\alpha_1}, S_{\alpha_1}) \).
\end{definition}

It is straightforward to verify that \( \sim \) is an equivalence relation. To compute the quotient element resulting from the division described earlier, it is important to introduce a key element of \(\mathsf{SMult}(\mathcal{G})\).

\paragraph{Graph \(\mathcal{G}\) as a monoidal element of \(\mathsf{SMult}(\mathcal{G})\).}
Let \( \mathsf{DE}(\mathcal{G}) \) denote the set of all directed edges in \( \mathcal{G} \). Then the element
{\small\[
\left( \bigoplus_{e \in \mathsf{DE}(\mathcal{G})} e, \mathsf{DE}(\mathcal{G}) \right) \in \mathsf{SMult}(\mathcal{G})
\]}
can be viewed as a monoidal representation of the graph \( \mathcal{G} \) itself. We denote this element by \( \mathcal{G} \).
The following theorem establishes the relationship between this monoidal element and the quotient structure obtained from \( \mathsf{Cov}(m) \):

\begin{theorem}\label{quotient}
    {\small\[
    \left( \bigoplus_{(\mathcal{M}_\alpha, S_\alpha) \in \mathsf{Cov}(m)} \mathcal{M}_\alpha, \bigcup_{(\mathcal{M}_\alpha, S_\alpha) \in \mathsf{Cov}(m)} \nu(S_\alpha) \right) / \sim\ 
    \cong \mathcal{G} \bullet \cdots \bullet \mathcal{G}
    \]}
    where \( \mathcal{G} \) appears \( m \) times.
\end{theorem}

We denote \( \mathcal{G} \bullet \cdots \bullet \mathcal{G} \) (with \( m \) repetitions) by \( \mathcal{G}^m \). Every element of \( \mathsf{Cov}(m) \) has a corresponding image in \( \mathcal{G}^m \), which therefore encodes a projection of each node's point of view.  
The element \( \mathcal{G}^m \) can be understood as a growing organism, where all paths \( e_1 \bullet \cdots \bullet e_m \) simultaneously evolve to length \( m \).
\subsection{Quantification of points of view; an enriched graph representation}

The monoidal homomorphism \(\mathsf{Tr}\), as defined in the GGNN framework, maps elements of \(\mathsf{Mod}(\mathcal{G})\) to matrices by counting paths between nodes. However, since the element \(\mathcal{G}^m\) is not constructed from directed subgraphs, it does not belong to \(\mathsf{Mod}(\mathcal{G})\), and hence \(\mathsf{Tr}\) is not directly applicable.
Instead, we adopt the path-counting functionality of \(\mathsf{Tr}\) and apply it to 
\(\mathcal{G}^m\), as formalized in the following theorem.
\begin{theorem}\label{mat rep of element}
    Suppose \(\mathcal{G}^m = (\mathcal{M}, S) \in \mathsf{SMult}(\mathcal{G})\). The number of paths in \(S\) from \(v_i\) to \(v_j\) equals \((\mA \circ \cdots \circ \mA)_{i,j}\) (\(m\)-times), where \(\mA\) is the adjacency matrix of \(\mathcal{G}\).
\end{theorem}
\paragraph{An enriched graph representation}
Motivated by Theorem~\ref{mat rep of element}, we define the matrix interpretation of \(\mathcal{G}^m\) to be \(\mA \circ \cdots \circ \mA\) (\(m\)-times), where \(\mA\) is the adjacency matrix of \(\mathcal{G}\). We \textbf{denote this matrix by} \(\mathsf{MI}(m)\).
The matrix \(\mathsf{MI}(m)\) quantitatively encodes each node's point of view on the graph topology at level \(m\). Specifically, the \(i\)-th row captures how node \(v_i\) perceives its structural environment via counts of length-\(m\) paths and their subpaths to others. The entry \(\mathsf{MI}(m)_{i,j}\) reflects the strength of connection between \(v_i\) and \(v_j\), enabling \(v_i\) to compare its affinities (e.g., between \(v_j\) and \(v_k\)) using \(\mathsf{MI}(m)_{i,j}\) and \(\mathsf{MI}(m)_{i,k}\). As \(m\) increases, the representation becomes increasingly rich, capturing deeper structural insights and enabling more informed comparisons. As a consequence of the following theorem, $\mathsf{MI}(m)$ can be computed as \(\mathsf{MI}(m)=( \mI+\mA)^m-\mI=\sum_{k=1}^m\binom{m}{k}\mA^k\).
\begin{theorem}\label{binomial}
    For a matrix $\mB$, we have
    \(\mB\circ\cdots\circ \mB(\textit{m-times })=( \mI+\mB)^m-\mI=\sum_{k=1}^m\binom{m}{k}\mB^k\)
\end{theorem}

\section{Merging points of view of nodes with the distributions of node attributes}
In this section, we develop a method for incorporating the topology of an attributed graph \(\mathcal{G} = (V, E, X)\) into the analysis of a categorical distribution \(P\) over node attributes \(X\). Specifically, we construct quantitative representations of how each node "perceives" \(P\) through its structural position. This allows us to propose interpretable formulations for both \(P(\cdot \mid \mathcal{G})\) and \(P(\cdot \mid v_j)\). We also validate the approach under a natural structural assumption inspired by complete graphs.
\subsection{Quantification of Points of View on Distributions}
We incorporate the distribution \(P\) into the points of view of nodes to the topology by defining carefully selected weights for the directed edges, resulting in a quantification of points of view of nodes on $P$.

\begin{definition}\label{induced-weights}
    For an attributed graph \(\mathcal{G} = (V, E, X)\) with adjacency matrix \(\mA\) and a categorical distribution \(P\) of \(X\), where \(P(x_j) = P_j\) and \(0 \leq \theta\leq 1\), we define the following:
    \begin{itemize}
        \item The matrix of induced weights from \(P\) of degree \(\theta\) is a \(|V| \times |V|\) matrix \(\mW\), where
        \[
        \mW_{i,j} = 
        \begin{cases} 
        \frac{P_j}{1 - P_j} \prod_{r = 1}^{|V|} (1 - P_r)^\theta & \text{if } \mA_{i,j} = 1, \\
        0 & \text{if } \mA_{i,j} = 0.
        \end{cases}
        \]
        \item The \(P\)-distributional matrix interpretation of \(\mathsf{Cov}(m)\) of degree \(\theta\), denoted by \(\mathsf{DMI}(P, m, \theta)\), is given by \(\mW \circ \cdots \circ \mW\).
    \end{itemize}
\end{definition}

\paragraph{A more enriched graph representation}
The matrix \(\mathsf{DMI}(P, m, \theta)\) provides an enriched representation of a graph by fusing topological structure with the distribution of node attributes. It extends the matrix \(\mathsf{MI}(m)\) by incorporating the attribute distribution \(P\) into the adjacency structure through the weighted matrix \(\mW\), as defined earlier. According to Theorem~\ref{binomial}, this matrix is computed as:
\(
\mathsf{DMI}(P, m, \theta) = \sum_{k = 1}^{m} \binom{m}{k} \mW^k.
\)

In our search for effective weighting schemes, we primarily identified two useful forms: 
\(\frac{P_j}{1 - P_j}\) (corresponding to \(\theta = 0\)) and \(\prod_{r \ne j}(1 - P_r) - \prod_{r = 1}^{|V|}(1 - P_r)\) (corresponding to \(\theta = 1\)). While the former has a stronger theoretical grounding (see Theorem~\ref{dist on complete graphs}), the latter showed favorable empirical performance.
We therefore proposed a general formulation parameterized by \(0\leq\theta \leq 1\), which interpolates between these behaviors. The scalar factor \(\prod_{r=1}^{|V|}(1 - P_r)^\theta\) serves as a global modulation term. Since this quantity lies in the interval \((0, 1]\), increasing \(\theta\) results in a uniform contraction of all edge weights, thereby dampening the contribution of longer paths. Each entry of \(\mW^k\) is scaled by \((\prod_{r=1}^{|V|}(1 - P_r)^\theta)^k\), giving \(\theta\) precise control over the influence of paths of different lengths.
Thus, the parameter \(\theta\) enables the representation to interpolate between \emph{extroverted} viewpoints (small \(\theta\)) and \emph{introverted} viewpoints (large \(\theta\)).

\paragraph{Posterior distribution modeling via points of view}
We are now ready to propose an interpretation of the posterior distributions \(P(\cdot \mid\ \mathcal{G})\) and \(P(\cdot \mid\ v)\). Our approach defines approximations to these distributions using two tunable parameters: the level \(m\), and the degree \(\theta\). As \(m\) increases, the approximation inherits more information from the interaction between topology and node attributes distribution. 
The following definition formalizes what we term the \emph{point of view} of a node and of the entire graph with respect to the categorical distribution \(P\). These perspectives are modeled as normalized rows and row sums of the matrix \(\mathsf{DMI}(P, m, \theta)\), respectively, and serve as approximations of \(P(\cdot \mid\ v)\) and \(P(\cdot \mid\ \mathcal{G})\).

\begin{definition}\label{posterior}
For an attributed graph $\mathcal{G} = (V, E, X)$ together with a categorical distribution $P$ of $X$, let $\mM = \mathsf{DMI}(P, m, \theta)$. We define the point of view of node $v_i$ with respect to the distribution $P$ of level $m$ and degree $\theta$, denoted by \(\mathsf{pov}(v_i, P, m, \theta)\), as an approximation of \(P(\cdot \mid v_i)\), and the point of view of the graph $\mathcal{G}$ with respect to the distribution $P$ of level $m$ and degree $\theta$, denoted by \(\mathsf{pov}(\mathcal{G}, P, m, \theta)\), as an approximation of \(P(\cdot \mid \mathcal{G})\), to be the following categorical distributions of node attributes:
\[
\mathsf{pov}(v_i, P, m, \theta) = \frac{\mM_{i,:}}{\lVert \mM_{i,:} \rVert_1}
\]
\[\mathsf{pov}(\mathcal{G}, P, m, \theta) = \frac{\sum_{i=1}^{|V|}\mathsf{pov}(v_i, P, m, \theta)}{|V|}\]
The \textbf{matrix of points of view}, denoted by \(\mathsf{POV}(P, m, \theta)\), is the matrix whose $i$-th row equals \(\mathsf{pov}(v_i, P, m, \theta)\).
\end{definition}
Utilizing the distributions $\mathsf{pov}(v_i,P,m,\theta)$ and $\mathsf{pov}(\mathcal{G},P,m,\theta)$, we extend the concept of the \textit{mean} for attributed graphs, which plays a central role in the design of our model for GAD.

\begin{definition}\label{affected_mean}
    For an attributed graph $\mathcal{G}=(V,E,X)$ together with a categorical distribution $P$ of $X$, we define the mean of $X$ from the point of view of $v_i$ of level $m$ and degree $\theta$, denoted by \(\mathsf{Mean}(v_i,P,m,\theta)\), and the mean of $X$ from the point of view of $\mathcal{G}$ of level $m$ and degree $\theta$, denoted by \(\mathsf{Mean}(\mathcal{G},P,m,\theta)\) as:
    \[\mathsf{Mean}(v_i,P,m,\theta)=\sum_{k=1}^{|V|}[\mathsf{pov}(v_i,P,m,\theta)]_{_k} x_k\]
    \[\mathsf{Mean}(\mathcal{G},P,m,\theta)=\sum_{k=1}^{|V|}[\mathsf{pov}(\mathcal{G},P,m,\theta)]_{_k} x_k\]
\end{definition}
The following example provides an intuitive illustration of point-of-view operators by showing how an agent can reason about a rumor source using only the topology of the graph.

\begin{figure}
    \centering
  \includegraphics[scale=0.55]{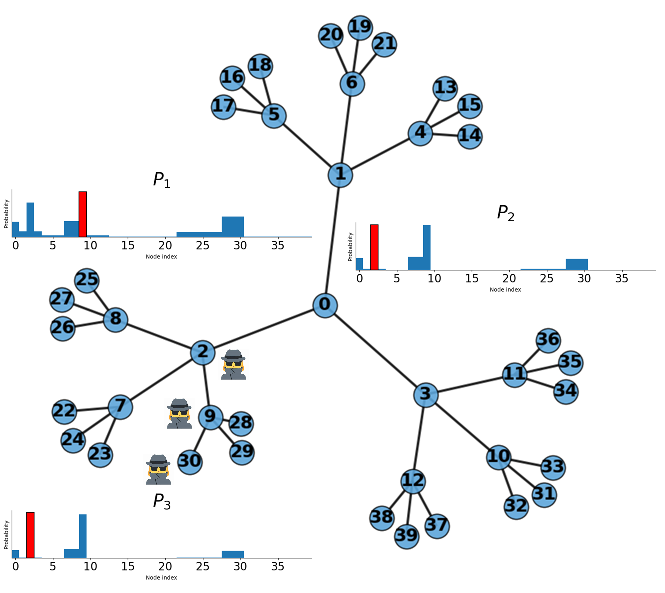}
   \caption{\textbf{Iterative localization of a rumor source via point of view}: Starting from a uniform belief, the agent repeatedly applies $\mathsf{pov}$, producing distributions $P_1, P_2, P_3$. At each stage, the agent moves to the maximizer of the current distribution (from $v_{30}$ to $v_9$, then to $v_2$), where the process stabilizes.
}
    \label{rumor}
\end{figure}

\begin{example}[Rumor source detection via point of view]

Figure \ref{rumor} shows a graph composed of several clusters connected by articulation nodes. 
Suppose an agent observes a rumor at node $v_{30}$ and wants to identify its source. 
Intuitively, the source is more likely to lie within the same region of the graph than in distant clusters.

\smallskip
\noindent\textbf{Stage 1.}
The agent starts with a discrete uniform belief $Q_1$ over all nodes. 
At node $v_{30}$, it computes
\[
P_1=\mathsf{pov}(v_{30}, Q_1, 70, 0).
\]
From $P_1$, the agent selects the next node by maximizing the distribution and finds $v_9$ to be the most likely source.

\smallskip
\noindent\textbf{Stage 2.}
The agent moves to $v_9$ and updates its belief and the associated point of view by computing
\[
Q_2 = \mathsf{pov}(v_9, Q_1, 70, 0), 
\qquad 
P_2=\mathsf{pov}(v_9, Q_2, 70, 0).
\]
From $P_2$, the agent again selects the next node and identifies $v_2$ as the most likely source.

\smallskip
\noindent\textbf{Stage 3.}
The agent moves to $v_2$ and computes
\[
Q_3 = \mathsf{pov}(v_2, Q_2, 70, 0), 
\qquad 
P_3=\mathsf{pov}(v_2, Q_3, 70, 0).
\]
The maximum of $P_3$ remains at $v_2$, so the process stabilizes and the agent reports $v_2$ as the most suspicious node.

\smallskip
This result matches the topology: since the rumor is only observed in the subgraph containing $v_{30}$, nodes in other clusters are less likely to be the source. In particular, although $v_0$ is a central articulation node, it lies outside the observed cluster and is therefore less suspicious than $v_2$.

\end{example}

\subsection{Sufficiency of proposed posterior distributions}\label{sec-suff}
We have proposed $\mathsf{pov}(v_i, P, m, \theta)$ and $\mathsf{pov}(\mathcal{G}, P, m, \theta)$ as approximations to $P(\cdot \mid v_i)$ and $P(\cdot \mid \mathcal{G})$, respectively. To evaluate the sufficiency of our approach, we consider the case of complete graphs as a natural and principled benchmark.
In a complete graph, every node is connected to every other node. This uniform and symmetric structure ensures that no individual node plays a privileged topological role. Thus, such graphs are topologically indistinguishable from sets when it comes to information aggregation. Consequently, the topology of a complete graph should not influence the distribution of node attributes.

\textbf{Assumption.} \emph{The topology of a complete graph does not affect the distribution of node attributes.}

This assumption serves as a boundary condition for evaluating the fidelity of our approach. Since a complete graph provides no meaningful structural constraints, a topology-aware approach that respects this condition should reproduce the original attribute distribution with minimal distortion.
Our approach involves transforming the topology of a graph into a form that interacts with the distribution of node attributes. It is essential that this transformation introduces only information related to the topology, without altering the statistical structure of the attributes in unexpected ways.
The following theorem confirms that our method respects this principle. As the level parameter \(m\) increases, the point of view distributions for nodes and the entire graph converge to the true attribute distribution:

\begin{theorem}\label{dist on complete graphs}
Let $\mathcal{G} = (V, E, X)$ be a complete attributed graph, and let $P$ be a categorical distribution over $X$, where $P(x_i) = P_i$ for each $x_i \in X$. Then for every node $v \in V$,
\[
\lim_{m \to \infty} \mathsf{pov}(v, P, m, 0) = (P_1, P_2, \dots, P_{|V|}), \quad \text{and}\]
\[\lim_{m \to \infty} \mathsf{pov}(\mathcal{G}, P, m, 0) = (P_1, P_2, \dots, P_{|V|}),
\]
\end{theorem}
To further substantiate our approach, we show that the desirable behavior established in Theorem~\ref{dist on complete graphs} for complete graphs extends naturally to arbitrary graphs. Since $\mathcal{G}^m$ plays a central role in our framework, we demonstrate this extension by emphasizing the structural relationship between elements $\mathcal{G}^m \in \mathsf{SMult}(\mathcal{G})$ and $\mathcal{C}^m \in \mathsf{SMult}(\mathcal{C})$, where $\mathcal{C} \supseteq \mathcal{G}$ denotes the complete graph obtained by adding all missing edges to $\mathcal{G}$.
This relationship is formalized through a \textit{Galois connection} between $\mathsf{SMult}(\mathcal{G})$ and $\mathsf{SMult}(\mathcal{C})$. We first show that the monoid $\mathsf{SMult}(\mathcal{G})$ admits a partial order:
\begin{theorem}\label{partial-order}
    For an arbitrary graph $\mathcal{G}$, the relation $\le$ on $\mathsf{SMult}(\mathcal{G})$ defined as follows:
    \[(\mathcal{M},S) \le (\mathcal{N},T) \iff \exists \iota:\mathcal{M} \hookrightarrow \mathcal{N}, \ \iota(S) \subseteq T\]
    is a partial order on $\mathsf{SMult}(\mathcal{G})$, compatible with the monoidal operation $\bullet$.
\end{theorem}
Next, we connect the monoids $\mathsf{SMult}(\mathcal{C})$ and $\mathsf{SMult}(\mathcal{G})$ through two order-preserving morphisms. We define the inclusion $\mathbf{In}:\mathsf{SMult}(\mathcal{G}) \to \mathsf{SMult}(\mathcal{C})$, and for the reverse direction, we define the mapping $\mathsf{R}$ as follows:
\begin{theorem}\label{mono homo R}
The mapping $\mathsf{R} : \mathsf{SMult}(\mathcal{C}) \to \mathsf{SMult}(\mathcal{G})$, defined by $\mathsf R(\mathcal{M}, S)=(\hat{\mathcal{M}}, \hat{S})$, is a monoidal homomorphism, where $\hat{\mathcal{M}}$ and $\hat{S}$ are constructed from $\mathcal M$ and $S$ as follows:
\begin{itemize}
    \item $\hat{\mathcal{M}}$ is the directed multigraph obtained by removing all directed edges in $\mathsf{DE}(\mathcal{C}) \setminus \mathsf{DE}(\mathcal{G})$.
    \item $\hat{S}$ is the set obtained by removing paths that traverse any directed edges in $\mathsf{DE}(\mathcal{C}) \setminus \mathsf{DE}(\mathcal{G})$.
\end{itemize}
\end{theorem}
It is straightforward to verify that $\mathbf{In}$ and $\mathsf{R}$ preserve the order $\le$.
\begin{theorem}\label{galois}
    The inclusion $\mathbf{In}:\mathsf{SMult}(\mathcal{G}) \to \mathsf{SMult}(\mathcal{C})$ and the monoidal homomorphism $\mathsf{R}:\mathsf{SMult}(\mathcal{C}) \to \mathsf{SMult}(\mathcal{G})$ form a Galois connection, with $\mathbf{In}$ as the left adjoint of $\mathsf{R}$.
\end{theorem}
The most significant advantage of this Galois connection is the ability to approximate an element of $\mathsf{SMult}(\mathcal{C})$ in $\mathsf{SMult}(\mathcal{G})$ using $\mathsf{R}$. Specifically, for $(\mathcal{M},S) \in \mathsf{SMult}(\mathcal{G})$ and $(\mathcal{N},T) \in \mathsf{SMult}(\mathcal{C})$:
\begin{center}\small
    $(\mathcal{M},S) \le \mathsf{R}(\mathcal{N},T) \iff (\mathcal{M},S) = \mathbf{In}(\mathcal{M},S) \le (\mathcal{N},T)$
\end{center}
Since $\mathsf{R}$ is a monoidal homomorphism and $\mathsf{R}(\mathcal{C}) = \mathcal{G}$, we deduce that $\mathsf{R}(\mathcal{C}^m) = \mathcal{G}^m$. Thus, as stated in the following corollary, $\mathcal{G}^m$ approximates $\mathcal{C}^m$ in $\mathsf{SMult}(\mathcal{G})$.
\begin{corollary}\label{approximate element}
    $\mathcal{G}^m$ is the largest element in $\mathsf{SMult}(\mathcal{G})$ satisfying:
    \((\mathcal{M},S) \le \mathcal{G}^m \iff \mathbf{In}(\mathcal{M},S) \le \mathcal{C}^m\)
\end{corollary}

\section{Induced Distribution (ID): a simple testbed model}\label{section-id}
Graph Anomaly Detection (GAD) serves as an excellent testbed to demonstrate the efficacy of our method. The main objective of GAD is to analyze the interplay between the distribution of node attributes $X$ and the topology of the graph for an attributed graph $\mathcal{G} = (V, E, X)$. This aligns directly with the principles of our approach. Accordingly, we instantiate our framework with a minimal testbed model, termed \textit{Induced Distribution} (\textbf{ID}), for unsupervised GAD.
\subsection{Crafting the Induced Distribution (\textbf{ID}) model}
\textbf{ID} assumes a discrete uniform distribution over node attributes by default, where \( P(x_i) = \frac{1}{|V|} \) for \( x_i \in X \). Unlike \textbf{ID}, our approach-particularly in Definition \ref{induced-weights}-imposes no restriction on the distribution of node attributes and allows for an arbitrary categorical distribution, which can be used whenever available instead of the uniform one.
We present the \textbf{ID} method in algorithmic form.

\paragraph{Replacing the adjacency matrix with the \(\mathsf{POV}\) matrix}

The \textbf{ID} method replaces the adjacency matrix with the \( \mathsf{POV}(P, m, \theta) \) matrix, which encodes both topological and distributional information.
To have \( \mathsf{POV}(P, m, \theta) \), we need to compute \( \mathsf{DMI}(P, m, \theta) \) as detailed in Theorem~\ref{binomial} and \textbf{Algorithm}~\ref{algorithm_dmi}, with a time complexity of \( \mathcal{O}(m |E| |V|) \).
This matrix is computed once and reused throughout the model.

\paragraph{Graph auto-encoder} 
Using the following message-passing operator, we employ a graph auto-encoder on \( \mathsf{POV}(P, m, \theta) \) to reconstruct node attributes. For $m \geq 2$, the diagonal entries of \( \mathsf{POV}(P, m, \theta) \) are nonzero, ensuring that $v_i \in N(v_i)$. Consequently, $h^l_{v_i}$ inherits from $x_i$. $h^l_v$ can be interpreted as a local version of \( \mathsf{Mean}(v, P, m, \theta) \), (see Definition \ref{affected_mean}).
{\small\[
h^l_v = \frac{1}{|N(v)|} \sum_{u \in N(v)} h_u^{l-1}.
\]}
\textbf{Loss function \ } 
The loss function of the model is defined as:
\(
\mathsf{Loss} = \sum_{i=1}^{|V|} ||x_i - \hat{x}_i||_2,
\)
where \( \hat{x}_i \) is the reconstructed attribute for node $v_i$. Interpreting \( \hat{x}_i \) as a local mean, this loss function aims to reduce local variances. It identifies nodes whose reconstructed attributes resist alignment with their original attributes.

\textbf{Score function \ } 
The score function incorporates two measurements. The first, derived from the local perspective, considers the distance between node attributes and their associated local means. The second, based on the global perspective, evaluates the distance of local means from \( \mathsf{Mean}(\mathcal{G}, P, m, \theta) \), the global mean. A linear combination of these distances forms the score function.
\[
\mathsf{Score} := \gamma ||x_i - \hat{x}_i||_1 + \lambda ||\mathsf{Mean}(\mathcal{G}, P, m, \theta) - \hat{x}_i||_1
\]



\subsection{Experiments}
We evaluate our model, \textbf{ID}, on six real-world datasets with organic outliers, following the benchmark protocol in \cite{bond}. All datasets are available in the PyGOD library \cite{pygod}. These include two small graphs (\textbf{Disney}, \textbf{Books}), three medium-scale graphs (\textbf{Weibo}, \textbf{Reddit}, \textbf{Enron}), and one large-scale graph (\textbf{DGraph}). 
Table \ref{stat} (Appendix) summarizes their statistics.  This diverse collection enables us to assess both the ability of \textbf{ID} to capture inlier distributions in low-data regimes (Disney and Books) and its scalability to large-scale data (DGraph). The selected hyperparameters for each dataset are reported in Table \ref{hyperparameters} (Appendix).
We compare \textbf{ID} against a comprehensive set of baselines, including LOF \citep{lof}, IF \citep{if}, SCAN \citep{scan}, Radar \citep{radar}, ANOMALOUS \citep{anomalous}, MLPAE \citep{mlpae}, GCNAE \citep{gcnae}, DOMINANT \citep{dominant}, DONE \citep{done}, AdONE \citep{adone}, AnomalyDAE\citep{anomalydae}, GAAN \citep{gaan}, CONAD \citep{conad}, \textit{GAD-NR} \citep{gad-nr}, and \textit{CoCo} \citep{coco}.


\paragraph{Results}
We report performance in terms of ROC-AUC, averaged over different random seeds. Results are summarized in Table~\ref{results}. 
\textbf{ID} achieves state-of-the-art performance on five datasets: \textbf{Reddit}, \textbf{Disney}, \textbf{Books}, \textbf{Enron}, and \textbf{DGraph}, while remaining competitive on \textbf{Weibo}. On small datasets (\textbf{Disney} and \textbf{Books}), a sufficiently large $m$ allows \textbf{ID} to capture global attribute distributions, leading to strong performance. For medium and large datasets (\textbf{Reddit}, \textbf{Enron}, and \textbf{DGraph}), smaller values of $m$ preserve sensitivity to local structures, enabling robust anomaly detection.  
As noted in \cite{bond}, anomalies in \textbf{Weibo} exhibit both structural and contextual characteristics. This suggests that the discrete uniform distribution assumed in our model may not perfectly match the attribute distribution in \textbf{Weibo}, partially explaining the relatively lower performance.

\textbf{Scalability of \textbf{ID} \ }
The scalability of \textbf{ID} arises from replacing the adjacency matrix with the \( \mathsf{POV} \) matrix, which incorporates topological context and serves as an implicit feature extractor. For example, in \textbf{DGraph}, \( \mathsf{POV}(P, 4, \theta) \) produces 11,404,815 nonzero entries, compared to 4,300,999 in the adjacency matrix, thereby enriching structural representation while remaining computationally efficient. Moreover, since anomalies in large graphs typically stem from local patterns, a relatively small value of \( m \) suffices to capture meaningful structure. We set \( m=4 \) for \textbf{DGraph}, which reduces the cost of computing the \( \mathsf{POV} \) matrix (10.22 seconds on CPU; see Table~\ref{tab:runtime}, Appendix). Combined with the simplicity of the message passing operator, loss, and score functions, this design supports both efficiency and scalability.

\begin{table}[t]
\centering
\caption{ROC-AUC (\%) of GAD methods on six datasets, reported as mean ± std over different random seeds. Best results are in \textbf{bold}.}
\label{results}
\scriptsize      
\setlength{\tabcolsep}{2pt}  
\renewcommand{\arraystretch}{1.05}
\begin{tabular}{lcccccc}
\toprule
Dataset & Weibo & Reddit & Disney & Books & Enron & DGraph \\
\midrule
LOF & 56.5±0.0 & 57.2±0.0 & 47.9±0.0 & 36.5±0.0 & 46.4±0.0 & TLE \\
IF & 53.5±2.8 & 45.2±1.7 & 57.6±2.9 & 43.0±1.8 & 40.1±1.4 & 60.9±0.7 \\
MLPAE & 82.1±3.6 & 50.6±0.0 & 49.2±5.7 & 42.5±5.6 & 73.1±0.0 & 37.0±1.9 \\
SCAN & 63.7±5.6 & 49.9±0.3 & 50.5±4.0 & 49.8±1.7 & 52.8±3.4 & TLE \\
Radar & \textbf{98.9±0.1} & 54.9±1.2 & 51.8±0.0 & 52.8±0.0 & 80.8±0.0 & OOM\_C \\
ANOMALOUS & \textbf{98.9±0.1} & 54.9±5.6 & 51.8±0.0 & 52.8±0.0 & 80.8±0.0 & OOM\_C \\
GCNAE & 90.8±1.2 & 50.6±0.0 & 42.2±7.9 & 50.0±4.5 & 66.6±7.8 & 40.9±0.5 \\
DOMINANT & 85.0±14.6 & 56.0±0.2 & 47.1±4.5 & 50.1±5.0 & 73.1±8.9 & OOM\_C \\
DONE & 85.3±4.1 & 53.9±2.9 & 41.7±6.2 & 43.2±4.0 & 46.7±6.1 & OOM\_C \\
AdONE & 84.6±2.2 & 50.4±4.5 & 48.8±5.1 & 53.6±2.0 & 44.5±2.9 & OOM\_C \\
AnomalyDAE & 91.5±1.2 & 55.7±0.4 & 48.8±2.2 & 62.2±8.1 & 54.3±11.2 & OOM\_C \\
GAAN & 92.5±0.0 & 55.4±0.4 & 48.0±0.0 & 54.9±5.0 & 73.1±0.0 & OOM\_C \\
CONAD & 85.4±14.3 & 56.1±0.1 & 48.0±3.5 & 52.2±6.9 & 71.9±4.9 & 34.7±1.2 \\
GAD-NR & 87.7±5.4 & 58.0±1.7 & 76.8±2.8 & 65.7±5.0 & 80.9±3.0 & OOM\_C \\
CoCo & 97.0±0.1 & 61.9±0.5 & 84.3±1.1 & 71.6±0.9 & 85.4±1.8 & OOM\_C \\
\midrule
\textbf{Our Model} & 91.2±1.3 & \textbf{62.0±1.7} & \textbf{90.6±2.8} & \textbf{76.8±1.8} & \textbf{87.4±3.6} & \textbf{69.7±1.4} \\
\bottomrule
\end{tabular}
\end{table}

\textbf{Hyperparameter Analysis \ }  
We analyze the hyperparameters $(m, \gamma, \lambda)$, while $\theta$ follows the previous analysis.

~\textbf{Hyperparameters $\gamma$ and $\lambda$:} In small graphs with sufficiently large $m$, nodes acquire broader perspectives on the attribute distribution, making comparisons to the global mean more meaningful. Thus, $\gamma < \lambda$ yields better results. Conversely, in large graphs, nodes receive localized views, making $\gamma > \lambda$ more effective.  
We validate this behavior using three datasets: Disney (small), Books (medium), and DGraph (large). Figure \ref{fig:gamma_lambda_trend} visualizes the sensitivity of AUC to the trade-off between global \(\lambda\) and local \(\gamma\) aggregation across datasets of different scales (Disney, Books, DGraph). Full numeric values are provided in Table \ref{tab:gamma_lambda_all} (Appendix).

\begin{figure}[t]
    \centering
    \includegraphics[width=0.76\linewidth]{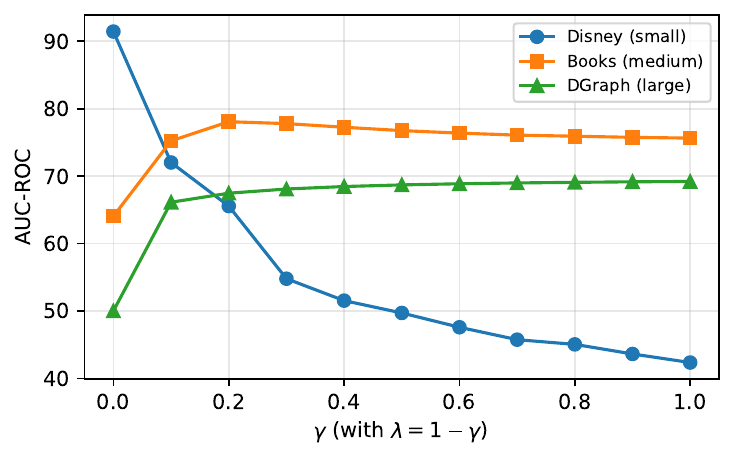}
    \caption{
        Sensitivity of AUC to $\gamma$ (with $\lambda = 1-\gamma$)
        for different graph scales 
    }
    \label{fig:gamma_lambda_trend}
\end{figure}

\textbf{Hyperparameter $m$:}
The hyperparameter $m$ controls the depth of each node’s point of view through
$\mathsf{Cov}(m)$, which aggregates information along longer paths. As shown in Figure~\ref{fig:m_ablation}, larger $m$ values improve both ROC-AUC and Average Precision (AP) on Disney and increase robustness to noisy attributes. This indicates that $m$ effectively balances local and global information in the graph. Full numeric values are provided in Table \ref{tab:m_ablation} (Appendix).
\begin{figure}[t]
    \centering
    \includegraphics[width=0.80\linewidth]{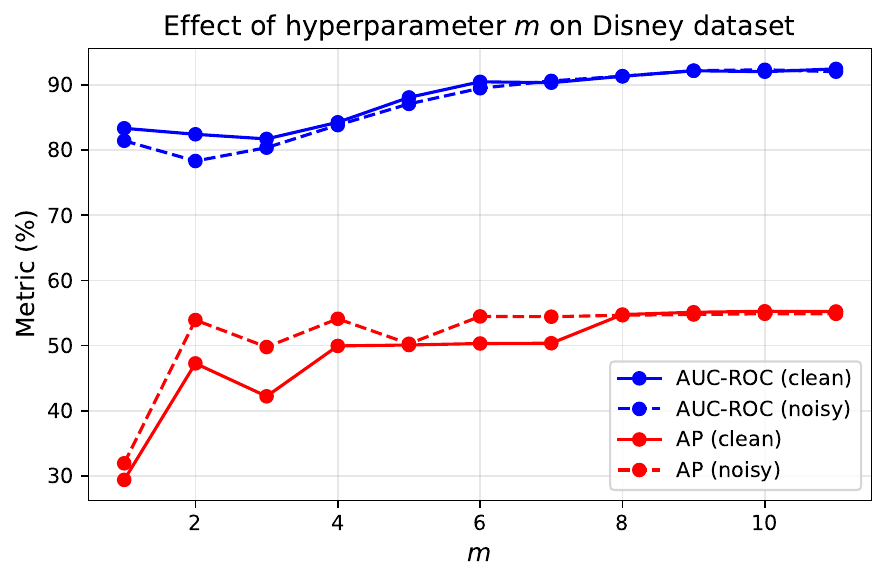}
    \caption{Increasing $m$ improves AUC-ROC, AP, and robustness to Gaussian noise ($\mu=0$, $\sigma^2=1$) on Disney
    }
    \label{fig:m_ablation}
\end{figure}
\section{Conclusion}
We introduced an algebraic--categorical framework for constructing topology-conditioned distributions that quantify how graph structure influences node attributes. By formalizing the \emph{point of view} of nodes, we obtained principled approximations of $P(\cdot \mid v)$ and $P(\cdot \mid \mathcal{G})$ together with a structural sufficiency guarantee. We instantiated this framework with a minimal testbed model, \textbf{ID}, and evaluated it via unsupervised GAD. Despite its simplicity, ID achieves strong empirical performance, highlighting the expressive power of topology-conditioned distributions and their potential for broader graph learning applications.

\bibliography{example_paper}

@book{highertopostheory,
title={Higher Topos Theory},
author={Lurie, Jacob},
volume={1},
year={2009},
publisher={Princeton University Press}
}

@book{commutativealgebra,
title={Commutative algebra: with a view toward algebraic geometry},
author={Eisenbud, David},
volume={1},
year={1995},
publisher={Springer-Verlag}
}

@book{abstractandconcretecategories,
title={Abstract and Concrete Categories – The Joy of Cats},
author={Adámek, Jiří and Herrlich, Horst and Strecker, George},
volume={1},
year={1990},
publisher={John Wiley and Sons}
}

@misc{GGNN,
      title={Grothendieck Graph Neural Networks Framework: An Algebraic Platform for Crafting Topology-Aware GNNs}, 
      author={Amirreza Shiralinasab Langari and Leila Yeganeh and Kim Khoa Nguyen},
      year={2024},
      eprint={2412.08835},
      archivePrefix={arXiv},
      primaryClass={cs.LG},
      url={https://arxiv.org/abs/2412.08835}, 
}

@book{maclane,
    author = {Saunders Mac{ }Lane},
    title = {Categories for the Working Mathematician},
    publisher = {Springer New York, NY},
    year = {1998}
}

@InProceedings{catdeeplearn1,
  title = 	 {Position: Categorical Deep Learning is an Algebraic Theory of All Architectures},
  author =       {Gavranovi\'{c}, Bruno and Lessard, Paul and Dudzik, Andrew Joseph and Von Glehn, Tamara and Madeira Ara\'{u}jo, Jo\~{a}o Guilherme and Veli\v{c}kovi\'{c}, Petar},
  booktitle = 	 {Proceedings of the 41st International Conference on Machine Learning},
  pages = 	 {15209--15241},
  year = 	 {2024},
  
  volume = 	 {235},
  series = 	 {Proceedings of Machine Learning Research},
  month = 	 {21--27 Jul},
  publisher =    {PMLR},
  
  url = 	 {https://proceedings.mlr.press/v235/gavranovic24a.html}
}

@misc{catdeeplearn2,
      title={Fundamental Components of Deep Learning: A category-theoretic approach}, 
      author={Bruno Gavranović},
      year={2024},
      eprint={2403.13001},
      archivePrefix={arXiv},
      primaryClass={cs.LG},
      url={https://arxiv.org/abs/2403.13001}, 
}

@misc{catfoundgrad,
      title={Categorical Foundations of Gradient-Based Learning}, 
      author={G. S. H. Cruttwell and Bruno Gavranović and Neil Ghani and Paul Wilson and Fabio Zanasi},
      year={2021},
      eprint={2103.01931},
      archivePrefix={arXiv},
      primaryClass={cs.LG},
      url={https://arxiv.org/abs/2103.01931}, 
}

@misc{catatten,
      title={On the Anatomy of Attention}, 
      author={Nikhil Khatri and Tuomas Laakkonen and Jonathon Liu and Vincent Wang-Maścianica},
      year={2024},
      eprint={2407.02423},
      archivePrefix={arXiv},
      primaryClass={cs.LG},
      url={https://arxiv.org/abs/2407.02423}, 
}

@misc{lenseslearners,
      title={Lenses and Learners}, 
      author={Brendan Fong and Michael Johnson},
      year={2019},
      eprint={1903.03671},
      archivePrefix={arXiv},
      primaryClass={cs.LG},
      url={https://arxiv.org/abs/1903.03671}, 
}

@inbook{dominant,
author = {Kaize Ding and Jundong Li and Rohit Bhanushali and Huan Liu},
title = {Deep Anomaly Detection on Attributed Networks},
booktitle = {Proceedings of the 2019 SIAM International Conference on Data Mining (SDM)},
chapter = {},
pages = {594-602},
doi = {10.1137/1.9781611975673.67},
URL = {https://epubs.siam.org/doi/abs/10.1137/1.9781611975673.67},
eprint = {https://epubs.siam.org/doi/pdf/10.1137/1.9781611975673.67}
}

@inproceedings{pickchoose,
author = {Liu, Yang and Ao, Xiang and Qin, Zidi and Chi, Jianfeng and Feng, Jinghua and Yang, Hao and He, Qing},
title = {Pick and Choose: A GNN-based Imbalanced Learning Approach for Fraud Detection},
year = {2021},
isbn = {9781450383127},
publisher = {Association for Computing Machinery},
address = {New York, NY, USA},
url = {https://doi.org/10.1145/3442381.3449989},
doi = {10.1145/3442381.3449989},
booktitle = {Proceedings of the Web Conference 2021},
pages = {3168–3177},
numpages = {10},
location = {Ljubljana, Slovenia},
series = {WWW '21}
}

@inproceedings{gad-nr,
author = {Roy, Amit and Shu, Juan and Li, Jia and Yang, Carl and Elshocht, Olivier and Smeets, Jeroen and Li, Pan},
title = {GAD-NR: Graph Anomaly Detection via Neighborhood Reconstruction},
year = {2024},
isbn = {9798400703713},
publisher = {Association for Computing Machinery},
address = {New York, NY, USA},
url = {https://doi.org/10.1145/3616855.3635767},
doi = {10.1145/3616855.3635767},
booktitle = {Proceedings of the 17th ACM International Conference on Web Search and Data Mining},
pages = {576–585},
numpages = {10},
location = {Merida, Mexico},
series = {WSDM '24}
}

@inproceedings{trunc,
 author = {Qiao, Hezhe and Pang, Guansong},
 booktitle = {Advances in Neural Information Processing Systems},
 editor = {A. Oh and T. Naumann and A. Globerson and K. Saenko and M. Hardt and S. Levine},
 pages = {49490--49512},
 publisher = {Curran Associates, Inc.},
 title = {Truncated Affinity Maximization: One-class Homophily Modeling for Graph Anomaly Detection},
 volume = {36},
 year = {2023}
}

@inproceedings{comga,
author = {Luo, Xuexiong and Wu, Jia and Beheshti, Amin and Yang, Jian and Zhang, Xiankun and Wang, Yuan and Xue, Shan},
title = {ComGA: Community-Aware Attributed Graph Anomaly Detection},
year = {2022},
isbn = {9781450391320},
publisher = {Association for Computing Machinery},
address = {New York, NY, USA},
url = {https://doi.org/10.1145/3488560.3498389},
doi = {10.1145/3488560.3498389},
booktitle = {Proceedings of the Fifteenth ACM International Conference on Web Search and Data Mining},
pages = {657–665},
numpages = {9},
location = {Virtual Event, AZ, USA},
series = {WSDM '22}
}

@INPROCEEDINGS{anomalydae,
  author={Fan, Haoyi and Zhang, Fengbin and Li, Zuoyong},
  booktitle={ICASSP 2020 - 2020 IEEE International Conference on Acoustics, Speech and Signal Processing (ICASSP)}, 
  title={Anomalydae: Dual Autoencoder for Anomaly Detection on Attributed Networks}, 
  year={2020},
  volume={},
  number={},
  pages={5685-5689},
  doi={10.1109/ICASSP40776.2020.9053387}}

@book{hun,
    author = {Thomas W. Hungerford},
    title = {Algebra},
    publisher = {Springer New York, NY},
    year = {2011}
}

@inproceedings{bond,
 author = {Liu, Kay and Dou, Yingtong and Zhao, Yue and Ding, Xueying and Hu, Xiyang and Zhang, Ruitong and Ding, Kaize and Chen, Canyu and Peng, Hao and Shu, Kai and Sun, Lichao and Li, Jundong and Chen, George H and Jia, Zhihao and Yu, Philip S},
 booktitle = {Advances in Neural Information Processing Systems},
 editor = {S. Koyejo and S. Mohamed and A. Agarwal and D. Belgrave and K. Cho and A. Oh},
 pages = {27021--27035},
 publisher = {Curran Associates, Inc.},
 title = {BOND: Benchmarking Unsupervised Outlier Node Detection on Static Attributed Graphs},
 volume = {35},
 year = {2022}
}

@ARTICLE{coco,
  author={Wang, Ruidong and Xi, Liang and Zhang, Fengbin and Fan, Haoyi and Yu, Xu and Liu, Lei and Yu, Shui and Leung, Victor C. M.},
  journal={IEEE Transactions on Knowledge and Data Engineering}, 
  title={Context Correlation Discrepancy Analysis for Graph Anomaly Detection}, 
  year={2025},
  volume={37},
  number={1},
  pages={174-187},
  doi={10.1109/TKDE.2024.3488375}}

@article{pygod,
  author  = {Kay Liu and Yingtong Dou and Xueying Ding and Xiyang Hu and Ruitong Zhang and Hao Peng and Lichao Sun and Philip S. Yu},
  title   = {{PyGOD}: A {Python} Library for Graph Outlier Detection},
  journal = {Journal of Machine Learning Research},
  year    = {2024},
  volume  = {25},
  number  = {141},
  pages   = {1--9},
  url     = {http://jmlr.org/papers/v25/23-0963.html}
}

@article{unquan,
title = {A review of uncertainty quantification in deep learning: Techniques, applications and challenges},
journal = {Information Fusion},
volume = {76},
pages = {243-297},
year = {2021},
issn = {1566-2535},
doi = {https://doi.org/10.1016/j.inffus.2021.05.008},
url = {https://www.sciencedirect.com/science/article/pii/S1566253521001081},
author = {Moloud Abdar and Farhad Pourpanah and Sadiq Hussain and Dana Rezazadegan and Li Liu and Mohammad Ghavamzadeh and Paul Fieguth and Xiaochun Cao and Abbas Khosravi and U. Rajendra Acharya and Vladimir Makarenkov and Saeid Nahavandi},
}

@inproceedings{graphpost,
 author = {Stadler, Maximilian and Charpentier, Bertrand and Geisler, Simon and Z\"{u}gner, Daniel and G\"{u}nnemann, Stephan},
 booktitle = {Advances in Neural Information Processing Systems},
 editor = {M. Ranzato and A. Beygelzimer and Y. Dauphin and P.S. Liang and J. Wortman Vaughan},
 pages = {18033--18048},
 publisher = {Curran Associates, Inc.},
 title = {Graph Posterior Network: Bayesian Predictive Uncertainty for Node Classification},
 url = {https://proceedings.neurips.cc/paper_files/paper/2021/file/95b431e51fc53692913da5263c214162-Paper.pdf},
 volume = {34},
 year = {2021}
}

@InProceedings{confpre,
  title = 	 {Conformal Prediction Sets for Graph Neural Networks},
  author =       {H. Zargarbashi, Soroush and Antonelli, Simone and Bojchevski, Aleksandar},
  booktitle = 	 {Proceedings of the 40th International Conference on Machine Learning},
  pages = 	 {12292--12318},
  year = 	 {2023},
  editor = 	 {Krause, Andreas and Brunskill, Emma and Cho, Kyunghyun and Engelhardt, Barbara and Sabato, Sivan and Scarlett, Jonathan},
  volume = 	 {202},
  series = 	 {Proceedings of Machine Learning Research},
  month = 	 {23--29 Jul},
  publisher =    {PMLR},
  url = 	 {https://proceedings.mlr.press/v202/h-zargarbashi23a.html},
}

@article{gaan,
  title={Graph Attention Auto-Encoder for Anomaly Detection in Attributed Networks},
  author={Chen, Bowen and et al.},
  journal={IEEE TNNLS},
  year={2021}
}

@article{done,
  title={DONE: Dynamic outlier detection in attributed networks},
  author={Peng, Hao and et al.},
  journal={IEEE TKDE},
  year={2021}
}

@article{adone,
  title={AdONE: Anomaly detection in attributed networks via contrastive self-supervised learning},
  author={Liu, Cheng and et al.},
  journal={Neurocomputing},
  year={2022}
}

@inproceedings{radar,
  title={Radar: Residual analysis for anomaly detection in attributed networks},
  author={Li, Jiawei and et al.},
  booktitle={IJCAI},
  year={2020}
}

@inproceedings{lof,
author = {Breunig, Markus M. and Kriegel, Hans-Peter and Ng, Raymond T. and Sander, J\"{o}rg},
title = {LOF: identifying density-based local outliers},
year = {2000},
isbn = {1581132174},
publisher = {Association for Computing Machinery},
address = {New York, NY, USA},
url = {https://doi.org/10.1145/342009.335388},
doi = {10.1145/342009.335388},
booktitle = {Proceedings of the 2000 ACM SIGMOD International Conference on Management of Data},
pages = {93–104},
numpages = {12},
location = {Dallas, Texas, USA},
series = {SIGMOD '00}
}

@article{if,
author = {Liu, Fei Tony and Ting, Kai Ming and Zhou, Zhi-Hua},
title = {Isolation-Based Anomaly Detection},
year = {2012},
issue_date = {March 2012},
publisher = {Association for Computing Machinery},
address = {New York, NY, USA},
volume = {6},
number = {1},
issn = {1556-4681},
url = {https://doi.org/10.1145/2133360.2133363},
doi = {10.1145/2133360.2133363},
journal = {ACM Trans. Knowl. Discov. Data},
month = mar,
articleno = {3},
numpages = {39}
}

@inproceedings{mlpae,
author = {Sakurada, Mayu and Yairi, Takehisa},
title = {Anomaly Detection Using Autoencoders with Nonlinear Dimensionality Reduction},
year = {2014},
isbn = {9781450331593},
publisher = {Association for Computing Machinery},
address = {New York, NY, USA},
url = {https://doi.org/10.1145/2689746.2689747},
doi = {10.1145/2689746.2689747},
booktitle = {Proceedings of the MLSDA 2014 2nd Workshop on Machine Learning for Sensory Data Analysis},
pages = {4–11},
numpages = {8},
location = {Gold Coast, Australia QLD, Australia},
series = {MLSDA'14}
}

@inproceedings{scan,
author = {Xu, Xiaowei and Yuruk, Nurcan and Feng, Zhidan and Schweiger, Thomas A. J.},
title = {SCAN: a structural clustering algorithm for networks},
year = {2007},
isbn = {9781595936097},
publisher = {Association for Computing Machinery},
address = {New York, NY, USA},
url = {https://doi.org/10.1145/1281192.1281280},
doi = {10.1145/1281192.1281280},
booktitle = {Proceedings of the 13th ACM SIGKDD International Conference on Knowledge Discovery and Data Mining},
pages = {824–833},
numpages = {10},
location = {San Jose, California, USA},
series = {KDD '07}
}

@inproceedings{anomalous,
author = {Peng, Zhen and Luo, Minnan and Li, Jundong and Liu, Huan and Zheng, Qinghua},
title = {ANOMALOUS: a joint modeling approach for anomaly detection on attributed networks},
year = {2018},
isbn = {9780999241127},
publisher = {AAAI Press},
booktitle = {Proceedings of the 27th International Joint Conference on Artificial Intelligence},
pages = {3513–3519},
numpages = {7},
location = {Stockholm, Sweden},
series = {IJCAI'18}
}

@misc{gcnae,
      title={Variational Graph Auto-Encoders}, 
      author={Thomas N. Kipf and Max Welling},
      year={2016},
      eprint={1611.07308},
      archivePrefix={arXiv},
      primaryClass={stat.ML},
      url={https://arxiv.org/abs/1611.07308}, 
}

@InProceedings{conad,
author="Xu, Zhiming
and Huang, Xiao
and Zhao, Yue
and Dong, Yushun
and Li, Jundong",
editor="Gama, Jo{\~a}o
and Li, Tianrui
and Yu, Yang
and Chen, Enhong
and Zheng, Yu
and Teng, Fei",
title="Contrastive Attributed Network Anomaly Detection with Data Augmentation",
booktitle="Advances in Knowledge Discovery and Data Mining",
year="2022",
publisher="Springer International Publishing",
address="Cham",
pages="444--457",
isbn="978-3-031-05936-0"
}
\bibliographystyle{icml2026}

\newpage
\appendix
\onecolumn
\section{Related Work}

\paragraph{Categorical deep learning} Category theory offers a unifying framework for the abstract structure of learning systems, and recent research has increasingly leveraged it to enhance interpretability and modularity in deep learning. In \cite{catdeeplearn1, catdeeplearn2}, category theory is employed to harmonize disparate deep learning architectures. \cite{catfoundgrad} introduces a categorical semantics for gradient-based learning, offering a uniform view of optimization algorithms. In \cite{catatten}, a categorical framework is developed to formalize attention mechanisms, while \cite{lenseslearners} demonstrates a connection between supervised learning and symmetric lenses, thus recasting learning as a bidirectional transformation process. Our approach builds on this tradition by integrating under categories and monoidal structures to model node viewpoints within attributed graphs, leading to probabilistic interpretations of structural context.

\paragraph{Graph anomaly detection} 
In unsupervised Graph Anomaly Detection (GAD), identifying nodes that deviate from expected graph patterns involves both structural and attribute-based signals. Graph Auto-Encoders (GAEs) form a foundational line of work in this domain. \cite{dominant} uses a GAE to reconstruct node attributes and topology, flagging high-reconstruction-error nodes as anomalies. Enhancements to GAE, such as attention-guided learning \cite{anomalydae}, selective substructure focus \cite{pickchoose}, and dual reconstruction objectives \cite{gad-nr}, have improved performance. 
Several models aim to better capture the interaction between topology and node attributes. \cite{trunc} introduces a topology-sensitive anomaly score, while \cite{comga} proposes community-level reconstruction rather than neighborhood-based models. Contrastive approaches like CoCo \cite{coco} integrate structural communities for self-supervised GAD, while GAAN \cite{gaan} applies attention over graph neighborhoods to better represent context. DONE \cite{done} and AdONE \cite{adone} use generative adversarial frameworks to learn normality distributions, and Radar \cite{radar} leverages residual analysis between structure and attributes. However, these methods typically treat structure and attributes as parallel inputs, without modeling how topology modifies attribute distributions.

Our proposed Induced Distribution (ID) model takes a distinct perspective: we use categorical and probabilistic tools to explicitly model how graph topology influences the node attribute distribution. By defining posterior distributions $P(\cdot \mid v)$ and $P(\cdot \mid \mathcal{G})$ through monoidal representations and under categories, our framework directly captures the topological modulation of attributes, a direction unexplored by existing GAD models.

\section{Fundamental Concepts in Category Theory}\label{cat}
We provide the fundamental categorical concepts that are used throughout the paper.
\begin{definition}\cite{hun}
    A monoid is a non-empty set $\mathsf{M}$ together with a binary operation $\cdot$ on $\mathsf{M}$ which
    \begin{itemize}
        \item[1)] is associative: $a\cdot(b\cdot c)=(a\cdot b)\cdot c$ for all $a,b,c\in \mathsf{M}$ and

   \item[2)] contains identity element $e\in\mathsf{M}$ such that $a\cdot e=e\cdot a=a$
    \end{itemize}
   If, for all $a, b \in \mathsf{M}$, the operation satisfies $a \cdot b = b \cdot a$, then we say that $\mathsf{M}$ is a commutative monoid.
\end{definition}
\begin{definition}\cite{abstractandconcretecategories}
    A category is a quadruple $\mathcal A = (\mathcal O,\text{hom},id,\cdot)$ consisting of
    \begin{itemize}
        \item a class $\mathcal O$, whose members are called objects,
        \item for each pair $(A,B)$ of objects, a set $\text{hom}(A,B)$, whose members are called morphisms from $A$ to $B$-[the statement “$f\in \text{hom}(A,B)$” is expressed
 more graphically by using arrows; e.g., by statements such as “$f:A\rightarrow B$ is a
 morphism” or “$\xymatrix{A\ar[r]^{f}&B}$ is a morphism”]
 \item for each object $A$, a morphism $id_A:A\rightarrow A$, called the identity on $A$,
 \item a composition law associating with each morphism $\xymatrix{A\ar[r]^f&B}$ and each morphism $\xymatrix{B\ar[r]^g&C}$ an morphism $\xymatrix{A\ar[r]^{g\cdot f}&B}$, called the composite of $f$ and $g$, subject to the following conditions:
 \begin{itemize}
     \item composition is associative; i.e., for morphisms $\xymatrix{A\ar[r]^f&B}$, $\xymatrix{B\ar[r]^g&C}$ and $\xymatrix{C\ar[r]^h&D}$
     the
 equation $h\cdot (g \cdot f) = (h\cdot g)\cdot f$ holds,
 \item identities act as identities with respect to composition; i.e., for morphisms $\xymatrix{A\ar[r]^f&B}$, we have $id_B\cdot f=f$ and $f\cdot id_A=f$,
 \item the sets $\text{hom}(A,B)$ are pairwise disjoint.
 \end{itemize}
    \end{itemize}
\end{definition}

\begin{definition} \cite{highertopostheory}
    For a category $\mathcal C$ and an object $C\in\mathcal C$, the under category $C/\mathcal C$ is the category that has morphisms with domain $C$, $C\rightarrow D$, as objects and the commutative triangles
    \[\xymatrix{&C\ar[ld]_{f}\ar[rd]^{g}&\\
    D\ar[rr]^{h}&&D'}\]
    as morphisms.
\end{definition}

\begin{definition}\cite{abstractandconcretecategories}
    For categories $\mathcal A$ and $\mathcal B$, a functor $F$ from $\mathcal A$ to $\mathcal B$ is a function that assigns
 to each $\mathcal A$-object $A$ a $\mathcal B$-object $F(A)$, and to each $\mathcal A$-morphism $\xymatrix{A\ar[r]^f&B}$ a $\mathcal B$-morphism $\xymatrix{F(A)\ar[r]^{F(f)}&F(B)}$, in such a way that
 \begin{itemize}
     \item $F$ preserves composition; i.e., $F(f\cdot g)$ = $F(f)\cdot F(g)$ whenever $f\cdot g$ is defined, and
     \item $F$ preserves identity morphisms; i.e., $F(id_A) = id_{F(A)}$ for each $\mathcal A$-object $A$.
 \end{itemize}
\end{definition}

\section{Algorithm}\label{algorithm}

\begin{algorithm}[H]
\caption{Computing $\mathsf{DMI}(P,m,\theta)$}
\label{algorithm_dmi}
\begin{algorithmic}[1]
\STATE \textbf{Input:} Adjacency matrix $\mA$, probability distribution $P$ of node attributes, integer $m$, parameter $\theta$
\STATE \textbf{Output:} $\mathsf{DMI}(P,m,\theta)$
\STATE \textbf{Initialization:} $\mW \gets \mathbf{0}$

\FOR{all $(i,j)$ such that $\mA_{i,j}\neq 0$}
    \STATE $\mW_{i,j}\gets \dfrac{P_j}{1-P_j}\cdot \prod_{r = 1}^{|V|} (1 - P_r)^\theta$
\ENDFOR

\STATE $\mathsf{DMI}(P,m,\theta)\gets(\mW+\mI)$
\FOR{$i = 2, \ldots, m$}
    \STATE $\mathsf{DMI}(P,m,\theta)\gets \mathsf{DMI}(P,m,\theta) + \mathsf{DMI}(P,m,\theta)\cdot\mW$
\ENDFOR
\STATE $\mathsf{DMI}(P,m,\theta)\gets \mathsf{DMI}(P,m,\theta)-\mI$
\STATE \textbf{Return:} $\mathsf{DMI}(P,m,\theta)$
\end{algorithmic}
\end{algorithm}
\subsection{Time Complexity}

We report the dominant time complexity of several graph anomaly detection models, based on the descriptions in their original papers. Here, $|V|$ denotes the number of nodes, $|E|$ the number of edges, $d$ the feature dimension, and $m$ and $H$ are model-specific hyperparameters.

\begin{itemize}
    \item \textbf{Induced Distribution (ID):} $\mathcal{O}(m\,|E|\,|V|)$
    
    \item \textbf{Radar \cite{radar}:} $\mathcal{O}(d|V|^2) + \mathcal{O}(|V|^3)$
    
    \item \textbf{ANOMALOUS \cite{anomalous}:} $\mathcal{O}(d|V|^2)$
    
    \item \textbf{DOMINANT \cite{dominant}:} $\mathcal{O}(H d |E| + |V|^2)$
    
    \item \textbf{AnomalyDAE \cite{anomalydae}:} $\mathcal{O}(|V|^2 + d|E|)$
\end{itemize}

These expressions capture the dominant asymptotic terms and omit constant factors and lower-order contributions.

\section{Experiment details}\label{experiment}
\paragraph{Datasets and Hyperparameters}  
Table~\ref{stat} summarizes the statistics of the datasets on which we applied our model, including the number of nodes, edges, features, average degree, number of outliers, and outlier ratio. Table~\ref{hyperparameters} reports the hyperparameters selected for each dataset in our experiments, including learning rate, dropout, model-specific parameters, scheduler settings, and the values of $\gamma$ and $\lambda$.
All experiments were conducted on Google Colab using CPU and standard RAM.
The runtime of computing the \(\mathsf{POV}\) matrix for each dataset is reported in Table~\ref{tab:runtime}.

\paragraph{Detailed Hyperparameter Results.}
Tables~\ref{tab:gamma_lambda_all} and~\ref{tab:m_ablation} report the complete numeric results corresponding to Figures~\ref{fig:gamma_lambda_trend} and~\ref{fig:m_ablation}, respectively. Each table lists the exact AUC-ROC and Average Precision values for all tested settings of \(\gamma\), \(\lambda\) and \(m\) across the Disney, Books, and DGraph datasets. These values provide quantitative support for the sensitivity trends and robustness analyses discussed in the main text.
\begin{table}[!h]
\centering
\caption{
Statistics of benchmark datasets used in experiments.
Degree denotes the average node degree; Ratio indicates the percentage of outliers.
}
\label{stat}
\setlength{\tabcolsep}{6pt}  
\begin{tabular}{lcccccc}
\toprule
Dataset & \#Nodes & \#Edges & \#Feat. & Deg. & \#Out. & Ratio \\
\midrule
Weibo  & 8{,}405      & 407{,}963    & 400 & 48.5 & 868     & 10.3\% \\
Reddit & 10{,}984     & 168{,}016    & 64  & 15.3 & 366     & 3.3\%  \\
Disney & 124          & 335          & 28  & 2.7  & 6       & 4.8\%  \\
Books  & 1{,}418      & 3{,}695      & 21  & 2.6  & 28      & 2.0\%  \\
Enron  & 13{,}533     & 176{,}987    & 18  & 13.1 & 5       & 0.4\%  \\
DGraph & 3.7M         & 4.3M         & 17  & 1.2  & 15{,}509 & 0.4\%  \\
\bottomrule
\end{tabular}
\end{table}

\begin{table}[!h]
\centering
\caption{
Hyperparameters used for each dataset
}
\label{hyperparameters}
\setlength{\tabcolsep}{4pt}
\begin{tabular}{lcccccc}
\toprule
 & Weibo & Reddit & Disney & Books & Enron & DGraph \\
\midrule
LR & 2e$^{-2}$ & 9e$^{-2}$ & 1e$^{-4}$ & 1e$^{-2}$ & 5e$^{-4}$ & 3e$^{-2}$ \\
Dropout & – & 0.7 & 0.9 & 0.4 & 0.5 & 0.0 \\
($m$, $\theta$) & (3, 1) & (3, 1) & (11, 1) & (5, 1) & (2, 1) & (4, 0) \\
Sched. (step, $\gamma$) & – & – & – & – & (4, 0.9) & – \\
($\gamma$, $\lambda$) & (1, 0) & (1, 0) & (0, 1) & (0.2, 0.8) & (0, 1) & (1, 0) \\
Hidden ch. & 32 & 18 & 48 & 30 & 8 & 17 \\
\bottomrule
\end{tabular}
\end{table}


\begin{table}[!h]
    \centering
    \caption{Runtime (in seconds) for computing \(\mathsf{POV}\) matrix on Google Colab using CPU and standard RAM}
    \begin{tabular}{lcccccc}
        \toprule
        \textbf{Dataset} & Weibo & Reddit & Disney & Books & Enron & DGraph \\
        \midrule
        \textbf{Runtime (s)} & 19.72 & 3.15 & 0.02 & 0.14 & 13.78 & 10.22 \\
        \bottomrule
    \end{tabular}
    \label{tab:runtime}
\end{table}


\begin{table}[!h]
\centering
\caption{Numeric results for the sensitivity analysis of \(\gamma\) and \(\lambda\) 
    (with \(\lambda = 1 - \gamma\)) across datasets of different scales. 
    Reported values correspond to the AUC-ROC and Average Precision scores 
    shown in Figure~\ref{fig:gamma_lambda_trend}.}
\label{tab:gamma_lambda_all}
\setlength{\tabcolsep}{5pt}
\begin{tabular}{c|ccc|ccc|ccc}
\toprule
 & \multicolumn{3}{c|}{\textbf{Disney (small)}} 
 & \multicolumn{3}{c|}{\textbf{Books (medium)}} 
 & \multicolumn{3}{c}{\textbf{DGraph (large)}} \\
\textbf{$\gamma$} & $\lambda$ & AUC &  & $\lambda$ & AUC &  & $\lambda$ & AUC &  \\
\midrule
0.00 & 1.00 & \textbf{91.45} & & 1.00 & 64.07 & & 1.00 & 50.02 \\
0.10 & 0.90 & 72.03 & & 0.90 & 75.24 & & 0.90 & 66.14 \\
0.20 & 0.80 & 65.57 & & 0.80 & \textbf{78.08} & & 0.80 & 67.48 \\
0.30 & 0.70 & 54.80 & & 0.70 & 77.80 & & 0.70 & 68.10 \\
0.40 & 0.60 & 51.55 & & 0.60 & 77.26 & & 0.60 & 68.47 \\
0.50 & 0.50 & 49.72 & & 0.50 & 76.76 & & 0.50 & 68.71 \\
0.60 & 0.40 & 47.60 & & 0.40 & 76.39 & & 0.40 & 68.88 \\
0.70 & 0.30 & 45.76 & & 0.30 & 76.09 & & 0.30 & 69.00 \\
0.80 & 0.20 & 45.06 & & 0.20 & 75.93 & & 0.20 & 69.09 \\
0.90 & 0.10 & 43.64 & & 0.10 & 75.76 & & 0.10 & 69.17 \\
1.00 & 0.00 & 42.37 & & 0.00 & 75.65 & & 0.00 & \textbf{69.22} \\
\bottomrule
\end{tabular}
\end{table}

\begin{table}[!h]
\centering
\caption{
Numeric results for the ablation of the hyperparameter \(m\) on the Disney and noisy Disney datasets, 
where Gaussian noise with mean \(\mu = 0\) and variance \(\sigma^2 = 1\) is added. 
These values correspond to the AUC-ROC and Average Precision trends presented in Figure~\ref{fig:m_ablation}.
}

\label{tab:m_ablation}
\setlength{\tabcolsep}{5pt}
\begin{tabular}{c|cc|cc}
\toprule
\multirow{2}{*}{$m$} & \multicolumn{2}{c|}{\textbf{Disney}} & \multicolumn{2}{c}{\textbf{Noisy Disney}} \\
 & AUC-ROC & AP & AUC-ROC & AP \\
\midrule
1  & 83.33 & 29.42 & 81.42 & 31.96 \\
2  & 82.41 & 47.28 & 78.31 & 53.93 \\
3  & 81.70 & 42.23 & 80.36 & 49.79 \\
4  & 84.25 & 49.96 & 83.82 & 54.12 \\
5  & 88.06 & 50.10 & 87.07 & 50.29 \\
6  & 90.46 & 50.33 & 89.47 & 54.48 \\
7  & 90.32 & 50.37 & 90.60 & 54.44 \\
8  & 91.31 & 54.77 & 91.31 & 54.66 \\
9  & 92.16 & 55.12 & 92.16 & 54.78 \\
10 & 92.01 & 55.27 & 92.30 & 54.89 \\
11 & \textbf{92.44} & \textbf{55.24} & 92.01 & 54.88 \\
\bottomrule
\end{tabular}
\end{table}

\newpage
\section{Proofs of the theorems}\label{sec-proofs}
\subsection*{Proof of Theorem \ref{cat(G)}}
\begin{proof}
The composition of two morphisms $e_1 \bullet \cdots \bullet e_k : u \to v$ and $d_1 \bullet \cdots \bullet d_l : v \to w$ is defined as follows:
\[
(e_1 \bullet \cdots \bullet e_k) \bullet (d_1 \bullet \cdots \bullet d_l) = e_1 \bullet \cdots \bullet e_k \bullet d_1 \bullet \cdots \bullet d_l : u \to w.
\]
The associativity of the composition follows directly from the associativity of the monoidal operation $\bullet$.

For each node $v$, we assign the morphism $0 : v \to v$, where $0$ is the identity element of the monoid $\mathsf{Mod}(\mathcal{G})$. This morphism is defined to be the identity morphism for the object (node) $v$. Since $0$ is the identity element of $\mathsf{Mod}(\mathcal{G})$, the following commutative diagrams hold:

\[
\begin{aligned}
    &\xymatrix{u \ar[rr]^{e_1 \bullet \cdots \bullet e_k} && v \\ u \ar[u]^{0} \ar[rru]_{e_1 \bullet \cdots \bullet e_k} &&}
    \hspace{1.5cm}
    \xymatrix{u \ar[rr]^{e_1 \bullet \cdots \bullet e_k} \ar[rrd]_{e_1 \bullet \cdots \bullet e_k} && v \ar[d]^{0} \\ && v.}
\end{aligned}
\]

The identity morphisms are compatible with the composition. Therefore, $\mathsf{Cat}(\mathcal{G})$ satisfies the axioms of a category. 
\end{proof}

\subsection*{Proof of Theorem \ref{iso_graphs_iso_cats}}
\begin{proof}
Let $f: \mathcal{G} \to \mathcal{H}$ be an isomorphism between the graphs $\mathcal{G}$ and $\mathcal{H}$. The isomorphism $f$ induces a monoidal isomorphism $\mathsf{Mod}(f): \mathsf{Mod}(\mathcal{G}) \to \mathsf{Mod}(\mathcal{H})$ (see \cite{GGNN}). We define the mapping $\mathsf{Cat}(f): \mathsf{Cat}(\mathcal{G}) \to \mathsf{Cat}(\mathcal{H})$, which assigns:
- To each object $v$, the object $f(v)$.
- To each morphism $e_1 \bullet \cdots \bullet e_m : u \to v$, the morphism $\mathsf{Mod}(f)(e_1 \bullet \cdots \bullet e_m) : f(u) \to f(v)$.

Since $\mathsf{Mod}(f)$ is a monoidal isomorphism, $\mathsf{Cat}(f)$ maps $0 : v \to v$ to $0 : f(v) \to f(v)$ and preserves composition. Therefore, $\mathsf{Cat}(f)$ is a functor.

Let $g: \mathcal{H} \to \mathcal{G}$ be the inverse of $f$. Then $\mathsf{Mod}(g)$ induces a functor from $\mathsf{Cat}(\mathcal{H})$ to $\mathsf{Cat}(\mathcal{G})$. It can be verified that this functor is the inverse of the functor $\mathsf{Cat}(f)$.

Now, let $L: \mathsf{Cat}(\mathcal{G}) \to \mathsf{Cat}(\mathcal{H})$ be an isomorphism with inverse $R$. This functor $L$ induces an isomorphism between the sets of nodes of $\mathcal{G}$ and $\mathcal{H}$. Let $e : u \to v$ be a directed edge, and suppose $L(e) = d_1 \bullet \cdots \bullet d_l$. Since $R$ preserves composition, we have:
\[
e = R(d_1) \bullet \cdots \bullet R(d_l).
\]
This implies that there exists an index $i_0$ such that $R(d_{i_0}) \neq 0$ and $R(d_i) = 0$ for $i \neq i_0$. Consequently, $d_i = 0$ for $i \neq i_0$. Thus, $L$ maps directed edges to directed edges, and $L(e) = d$ for some directed edge $d : L(u) \to L(v)$.

Similarly, $L$ maps the reverse edge $e' : v \to u$ to $d' : L(v) \to L(u)$. Therefore, $L$ induces a mapping between the sets of edges of $\mathcal{G}$ and $\mathcal{H}$ that is compatible with the induced mapping between the sets of nodes.
\end{proof}

\subsection*{Proof of Theorem \ref{quotient}}
\begin{proof}
    Based on the definition of $\sim$, the number of equivalence classes for every directed edge in
    \[\bigoplus_{(\mathcal{M}_\alpha, S_\alpha) \in \mathsf{Cov}(m)} \mathcal{M}_\alpha\]
    is less than or equal to $m$. We show that this number is exactly $m$.

    Let $e: v \rightarrow w$ be a directed edge, and let $e'$ be the directed edge with the opposite direction. Consider the element 
    \[
    e_1 \bullet e_2 \cdots \bullet e_m \in \mathsf{Cov}(m),
    \]
    where $e_i = e$ if $i$ is odd, and $e_i = e'$ if $i$ is even. For each odd number between $1$ and $m$, there is an equivalence class $[e]_i$. Similarly, we can construct the opposite of this element:
    \[
    e_1 \bullet e_2 \cdots \bullet e_m \in \mathsf{Cov}(m),
    \]
    where $e_i = e$ if $i$ is even, and $e_i = e'$ if $i$ is odd. For each even number between $1$ and $m$, there is an equivalence class $[e]_i$. Therefore, for every directed edge $e$, there are exactly $m$ equivalence classes $[e]_i$.

    Now, let $(\mathcal{M}, S) = \mathcal{G}_1 \bullet \mathcal{G}_2 \bullet \cdots \bullet \mathcal{G}_m$, where for all $1 \leq i \leq m$, $\mathcal{G}_i$ is the monoidal representation of $\mathcal{G}$:
    \[
    \mathcal{G}_i = \left(\bigoplus_{e \in \mathsf{DE}(\mathcal{G})} e, \mathsf{DE}(\mathcal{G}) \right) \in \mathsf{SMult}(\mathcal{G}).
    \]
    We define the mapping
    \[
    \phi: \left(\bigoplus_{(\mathcal{M}_\alpha, S_\alpha) \in \mathsf{Cov}(m)} \mathcal{M}_\alpha \right) / \sim \longrightarrow \mathcal{M}
    \]
    by sending $[e]_i$ to the directed edge representing $e \in \mathcal{G}_i$ in $\mathcal{M}$. If $[e]_i = [e']_i$, then $e$ and $e'$ start at the same node and end at the same node. Consequently, $\phi([e]_i) = \phi([e']_i)$, so $\phi$ is well-defined.

    The mapping $\phi$ has an inverse that sends the directed edge representing $e \in \mathcal{G}_i$ in $\mathcal{M}$ to $[e]_i$. Therefore, $\phi$ is an isomorphism between multigraphs.

    The equivalence relation $\sim$ induces an equivalence relation on
    \[
    \bigcup_{(\mathcal{M}_\alpha, S_\alpha) \in \mathsf{Cov}(m)} \nu(S_\alpha).
    \]
    Let $(\mathcal{M}_{\alpha_0}, S_{\alpha_0}) = e_1 \bullet \cdots \bullet e_m$ and $(\mathcal{M}_{\alpha_1}, S_{\alpha_1}) = c_1 \bullet \cdots \bullet c_m$. For a path $p = e_{p_1}, \cdots, e_{p_k} \in \nu(S_{\alpha_0})$, where $1 \leq p_1 < \cdots < p_k \leq m$, and a path $q = c_{q_1}, \cdots, c_{q_k} \in \nu(S_{\alpha_1})$, where $1 \leq q_1 < \cdots < q_k \leq m$, the induced relation is defined as follows:
    \[
    p \sim q \iff e_{p_i} \sim c_{q_i} \quad \text{for all } 1 \leq i \leq k.
    \]

    Then, a path $p \in \left(\bigcup_{(\mathcal{M}_\alpha, S_\alpha) \in \mathsf{Cov}(m)} \nu(S_\alpha) \right) / \sim$ is a sequence $[e_1]_{p_1}, [e_2]_{p_2}, \cdots, [e_k]_{p_k}$, where $k \leq m$ and $1 \leq p_1 < \cdots < p_k \leq m$. We extend the definition of $\phi$ to 
    \[
    \left(\bigcup_{(\mathcal{M}_\alpha, S_\alpha) \in \mathsf{Cov}(m)} \nu(S_\alpha) \right) / \sim
    \]
    as follows:
    \[
\phi([e_1]_{m_1} [e_2]_{m_2} \cdots [e_k]_{m_k}) := 
\phi([e_1]_{m_1})\, \phi([e_2]_{m_2}) \cdots \phi([e_k]_{m_k})\].

    Two paths $p, q \in \left(\bigcup_{(\mathcal{M}_\alpha, S_\alpha) \in \mathsf{Cov}(m)} \nu(S_\alpha) \right) / \sim$ are equal if and only if they traverse the same directed edges with the same equivalence classes, which happens if and only if $\phi(p) = \phi(q)$. Thus, $\phi$ is well-defined and injective. Clearly, $\phi$ is also surjective.

    Therefore, we conclude:
    \[
    \left(\bigoplus_{(\mathcal{M}_\alpha, S_\alpha) \in \mathsf{Cov}(m)} \mathcal{M}_\alpha, \bigcup_{(\mathcal{M}_\alpha, S_\alpha) \in \mathsf{Cov}(m)} \nu(S_\alpha) \right) / \sim \ \overset{\phi}{\cong} \ \mathcal{G}^m.
    \]
\end{proof}

\subsection*{Proof of Theorem \ref{mat rep of element}}
\begin{proof}
    We prove this statement by induction on $m$.
 
    For $m=1$, the claim is trivial since the number of paths from $v_i$ to $v_j$ in $(\bigoplus_{e \in \mathsf{DE}(\mathcal{G})} e, \mathsf{DE}(\mathcal{G}))$ corresponds exactly to the $(i,j)$-entry of $\mA$, the adjacency matrix of $\mathcal{G}$.
  
    Assume the statement is true for $m-1$, i.e., the number of paths from $v_i$ to $v_j$ in $\mathcal{G}^{m-1} = (\mathcal{N}, T)$ equals the $(i,j)$-entry of $\mA \circ \cdots \circ \mA$ (the composition of $\mA$ with itself ($(m-1)$-times).  
    We will show the statement holds for $m$.

    For $m$, let:
    \[
(\mathcal{M}, S) = (\mathcal{N}, T) \bullet 
\Big(\bigoplus_{e \in \mathsf{DE}(\mathcal{G})} e,\; \mathsf{DE}(\mathcal{G})\Big)
= \Big(\bigoplus_{e \in \mathsf{DE}(\mathcal{G})} e \oplus \mathcal{N},\;
T \star \mathsf{DE}(\mathcal{G})\Big).
\]

A path in $S = T \star \mathsf{DE}(\mathcal{G})$ can be categorized as follows:
    1. It is a path in $T$ (i.e., a path entirely within $\mathcal{G}^{m-1}$), or  
    2. It is a path in $\mathsf{DE}(\mathcal{G})$, or  
    3. It is a composition of a subpath in $T$ from $v_i$ to some intermediate node $v_l$, followed by a subpath in $\mathsf{DE}(\mathcal{G})$ from $v_l$ to $v_j$.

    Using this decomposition, the total number of paths from $v_i$ to $v_j$ in $S$ is given by:
    \text{(Number of paths in $T$ from $v_i$ to $v_j$)} + \text{(Number of paths in $\mathsf{DE}(\mathcal{G})$ from $v_i$ to $v_j$)} + \text{(Composition of subpaths from $v_i$ to $v_j$)}.

    By the inductive hypothesis, the number of paths in $T$ equals the $(i,j)$-entry of $\mA \circ \cdots \circ \mA$ ($(m-1)$ times). Additionally, the number of paths in $\mathsf{DE}(\mathcal{G})$ is given by $\mA$, and the composition of subpaths corresponds to the multiplication of $(m-1)$ copies of $\mA$ with $\mA$, i.e., $(\mA \circ \cdots \circ \mA)\mA$.

    Therefore, the total number of paths from $v_i$ to $v_j$ in $S$ equals the $(i,j)$-entry of:
\[
\mA + (\mA \circ \cdots \circ \mA)_{(m-1)\text{-times}} 
+ (\mA \circ \cdots \circ \mA)_{(m-1)\text{-times}} \mA 
= (\mA \circ \cdots \circ \mA)_{(m)\text{-times}}.
\]
Thus, the statement is true for $m$. By induction, the proof is complete.
\end{proof}

\subsection*{Proof of Theorem \ref{binomial}}
\begin{proof}
    In \cite{GGNN}, it is proved that for $\mA_1, \mA_2, \cdots, \mA_m\in \mathsf{Mat}_{n}(\mathbb{R})$ with $m\in \mathbb{N}$: 
    \begin{equation*}
    \begin{split}
    \mA_1\circ \mA_2 \circ \cdots \circ \mA_m & =\sum_{i=1}^{m}\mA_i\\
    &+\sum_{\sigma\in O(m,2)}\mA_{\sigma_1}\mA_{\sigma_2}+\cdots\\
    &+\sum_{\sigma\in O(m,j)}\mA_{\sigma_1}\cdots \mA_{\sigma_j}+ \cdots\\
    &+\mA_1 \mA_2\cdots \mA_m
    \end{split}
    \end{equation*}
    where $O(m, i)$ is the set of all strictly monotonically increasing sequences of $i$ numbers of $\lbrace1, \cdots, m\rbrace$.
    By setting $\mA_i=\mB$ for all $1\le i\le m$, we get
    \[\mB\circ\cdots\circ \mB(\ m\textit{-times })=\sum_{i=1}^m\binom{m}{i}\mB^i=( \mI+\mB)^m-\mI\]
    
\end{proof}
\subsection*{Proof of Theorem \ref{dist on complete graphs}}
Before proving the theorem, we need to state the following lemma. 

\begin{lemma}\label{lem}
    \[\frac{P_j}{1-P_j}x=P_jx+P_j\frac{P_j}{1-P_j}x\]
\end{lemma}

The proof of this lemma is obvious. In the following, we prove Theorem \ref{dist on complete graphs}.
\begin{proof}
Let \(\mathcal{C}\) be a complete graph. Associated with every path \(p = e_{m_1}, \dots, e_{m_k} \in \mathcal{C}^m\), we define \(\pi(p)\) as

\[
\pi(p) = \prod_{l=1}^{k} w(e_{m_l}),
\]

where \(w(e_{m_l})\) is the associated weight for the directed edge \(e_{m_l}\) in \(\mW\), the matrix of induced weights from \(P\) with \(\theta = 0\) (see Definition \ref{induced-weights}). Also, for every path \(p \in \mathcal{C}^m\), we associate \(\bar{p}\), which is the path in \(\mathcal{C}^m\) obtained by removing the last directed edge. For each pair of nodes \((v_i, v_j)\), we form the following monotone sequences of sets:
\[
\big\{\, \bar{p} : p \in \mathcal{C}^m 
\text{ is a path from } v_i \text{ to } v_j \,\big\}
\subseteq 
\big\{\, \bar{p} : p \in \mathcal{C}^{m+1} 
\text{ is a path from } v_i \text{ to } v_j \,\big\}.
\]

\[
\big\{\, p : p \in \mathcal{C}^m 
\text{ is a path from } v_i 
\text{ to a node in } V \setminus v_j \,\big\}
\subseteq
\big\{\, p : p \in \mathcal{C}^{m+1} 
\text{ is a path from } v_i 
\text{ to a node in } V \setminus v_j \,\big\}.
\]

Since \(\mathcal{C}\) is a complete graph, one can easily verify that the limits of both sequences, as \(m \to \infty\), are equal:

\begin{equation}\label{new1}
\bigcup_{m=1}^{\infty} 
\big\{\, \bar{p} : p \in \mathcal{C}^m 
\text{ is a path from } v_i \text{ to } v_j \,\big\}
= 
\bigcup_{m=1}^{\infty} 
\big\{\, p : p \in \mathcal{C}^m 
\text{ is a path from } v_i 
\text{ to a node in } V \setminus v_j \,\big\}.
\end{equation}

Let \(\mM^{(m)} = \mathsf{DMI}(P, m, 0)\). Thus, we have

\[
\mM^{(m)}_{i,j} = \sum_{\lbrace p \in \mathcal{C}^m: \text{ from } v_i \text{ to } v_j \rbrace} \pi(p).
\]

Since the end node of every path \(p\) in the above summation is \(v_j\), \(\pi(p)\) is a multiple of \(\frac{P_j}{1 - P_j}\). Lemma \ref{lem} implies
\[
\mM^{(m)}_{i,j}
= P_j \!\!\sum_{\{\bar{p}: p \in \mathcal{C}^m 
\text{ is a path from } v_i \text{ to } v_j\}}\! \pi(\bar{p})
+\, P_j \!\!\sum_{\{p \in \mathcal{C}^m :
\text{ from } v_i \text{ to } v_j\}}\! \pi(p).
\]

Therefore,

\begin{equation}\label{new2}
    \lim_{m \to \infty} \frac{\mM^{(m)}_{i,j}}{\|\mM^{(m)}_{i,:}\|_1} = P_j \lim_{m \to \infty} \frac{\sum_{\lbrace \bar{p}: p \in \mathcal{C}^m \text{ is a path from } v_i \text{ to } v_j \rbrace} \pi(\bar{p}) + \sum_{\lbrace p \in \mathcal{C}^m: \text{ from } v_i \text{ to } v_j \rbrace} \pi(p)}{\|\mM^{(m)}_{i,:}\|_1}.
\end{equation}

Using Equation \ref{new1}, Equation \ref{new2} becomes

\[
= P_j \lim_{m \to \infty} \frac{\sum_{\lbrace p: p \in \mathcal{C}^m \text{ is a path from } v_i \text{ to a node in } V \setminus v_j \rbrace} \pi(p) + \sum_{\lbrace p \in \mathcal{C}^m: \text{ from } v_i \text{ to } v_j \rbrace} \pi(p)}{\|\mM^{(m)}_{i,:}\|_1},
\]

\[
= P_j \lim_{m \to \infty} \frac{\sum_{\lbrace p: p \in \mathcal{C}^m \text{ is a path from } v_i \text{ to a node in } V \rbrace} \pi(p)}{\|\mM^{(m)}_{i,:}\|_1} = P_j \lim_{m \to \infty} \frac{\|\mM^{(m)}_{i,:}\|_1}{\|\mM^{(m)}_{i,:}\|_1} = P_j.
\]
\end{proof}
\subsection*{Proof of Theorem \ref{partial-order}}
\begin{proof}
Compatibility implies:
\begin{equation*}\begin{split}
            (\mathcal{M},S) \le (\mathcal{N},T) \implies (\mathcal{M},S) \bullet (\mathcal{Q},R) \le (\mathcal{N},T) \bullet (\mathcal{Q},R)& \\
            \And (\mathcal{Q}',R') \bullet (\mathcal{M},S) \le (\mathcal{Q}',R') \bullet (\mathcal{N},T)&
        \end{split}
\end{equation*}
    The proof of theorem is straightforward.
\end{proof}
\subsection*{Proof of Theorem \ref{mono homo R}}
\begin{proof}
    Let \((\mathcal{M}, S)\) and \((\mathcal{N}, T)\) be in \(\mathsf{SMult}(\mathcal{C})\). It is easy to verify that

\[
\widehat{\mathcal{M} \oplus \mathcal{N}} = \hat{\mathcal{M}} \oplus \hat{\mathcal{N}}.
\]

We have: \(r \in \widehat{S \star T}\) if and only if \(r \in S \star T\) and \(r\) does not traverse any directed edge of the set \(\mathsf{DE}(\mathcal{C}) \setminus \mathsf{DE}(\mathcal{G})\), if and only if \(r \in S\) or \(r \in T\), or there are \(p \in S\) and \(q \in T\) such that \(r = pq\), where \(p\) and \(q\) do not traverse any directed edge of the set \(\mathsf{DE}(\mathcal{C}) \setminus \mathsf{DE}(\mathcal{G})\), if and only if \(r \in \hat{S} \star \hat{T}\). Then,

\[
\widehat{S \star T} = \hat{S} \star \hat{T}.
\]

Therefore,

\[
\mathsf{R}((\mathcal{M}, S) \bullet (\mathcal{N}, T)) = \mathsf{R}(\mathcal{M}, S) \bullet \mathsf{R}(\mathcal{N}, T).
\]

\end{proof}

\subsection*{Proof of Theorem \ref{galois}}
\begin{proof}
The definition of \(\mathsf{R}\) is as follows:

\[
\mathsf{R}(\mathbf{In}(\mathcal{M}, S)) = \mathsf{R}(\mathcal{M}, S) = (\mathcal{M}, S)
\]

for \((\mathcal{M}, S) \in \mathsf{SMult}(\mathcal{G})\). Since \(\mathsf{R}\) preserves the relation \(\le\), for \((\mathcal{N}, T), (\mathcal{Q}, R) \in \mathsf{SMult}(\mathcal{C})\), we have:

\[
(\mathcal{N}, T) \le (\mathcal{Q}, R) \implies \mathsf{R}(\mathcal{N}, T) \le \mathsf{R}(\mathcal{Q}, R)
\]

Let \((\mathcal{M}, S) \le \mathsf{R}(\mathcal{N}, T)\), where \((\mathcal{M}, S) \in \mathsf{SMult}(\mathcal{G})\) and \((\mathcal{N}, T) \in \mathsf{SMult}(\mathcal{C})\). Then, there exists a morphism \(\iota: \mathcal{M} \hookrightarrow \hat{\mathcal{N}}\) such that \(\iota(S) \subseteq \hat{T}\). Consequently,

\[
\iota: \mathcal{M} \hookrightarrow \hat{\mathcal{N}} \hookrightarrow \mathcal{N} \quad \text{and} \quad \iota(S) \subseteq\hat{T} \subseteq T
\]

Thus, \((\mathcal{M}, S) \le (\mathcal{N}, T)\). Therefore,

\[
\mathbf{In}(\mathcal{M}, S) = (\mathcal{M}, S) \le (\mathcal{N}, T) \iff (\mathcal{M}, S) = \mathsf{R}(\mathbf{In}(\mathcal{M}, S)) \le \mathsf{R}(\mathcal{N}, T)
\]

Hence, \((\mathbf{In}, \mathsf{R})\) forms a Galois connection.
\end{proof}
\subsection{Proof of Corollary \ref{approximate element}}
\begin{proof}
    This is a direct consequence of Theorem \ref{galois}.
\end{proof}

\end{document}